\documentclass{article}

\usepackage[final,nonatbib]{nips_2016}

\usepackage[utf8]{inputenc} 
\usepackage[T1]{fontenc}    
\usepackage{hyperref}       
\usepackage{url}            
\usepackage{booktabs}       
\usepackage{amsfonts}       
\usepackage{nicefrac}       
\usepackage{microtype}      

\usepackage{amsmath,amsfonts,amsthm,amssymb}
\usepackage{algpseudocode}
\usepackage{algorithm}
\usepackage{footnote}
\usepackage{color}
\usepackage{graphicx}

\usepackage[titletoc]{appendix}

\def\Rbb{{\mathbb R}}
\newcommand\inner[1]{\ensuremath{\langle #1 \rangle}}

\newcommand{\vct}{\boldsymbol }
\newcommand{\mat}{\mathbf}

\newcommand{\nml}{\mathcal{N}}

\newcommand{\SPAN}{\mathrm{span}}

\newcommand{\argmax}{\mathrm{argmax}}

\newcommand{\range}{\mathrm{Range}}

\newcommand{\tr}{\mathrm{tr}}

\newcommand{\op}{\mathrm{op}}
\newcommand{\sg}{\mathcal{SG}}
\newcommand{\VEC}{\mathrm{vec}}

\newtheorem{thm}{Theorem}[section]
\newtheorem{lem}{Lemma}[section]
\newtheorem{cor}{Corollary}[section]
\newtheorem{prop}{Proposition}[section]
\newtheorem{asmp}{Assumption}[section]
\newtheorem{defn}{Definition}[section]

\usepackage[explicit]{titlesec}
\titlespacing{\paragraph}{0pt}{0.5em}{0.5em}[]

\allowdisplaybreaks[4]

\title{\scalebox{0.95}{Online and Differentially-Private Tensor Decomposition}}

\author{
  Yining Wang\\
  Machine Learning Department\\
  Carnegie Mellon University\\
  \texttt{yiningwa@cs.cmu.edu}
  \And
  Animashree Anandkumar\\
  Department of EECS\\
  University of California, Irvine\\
  \texttt{a.anandkumar@uci.edu}
}

\begin{document}

\maketitle

\begin{abstract}Tensor decomposition is an important tool for big data analysis. In this paper, we resolve many of the key algorithmic questions regarding robustness, memory efficiency, and differential  privacy of tensor decomposition. We propose simple variants of the tensor power method which enjoy these strong properties. We present the first guarantees for online tensor power method which has  a linear memory requirement. Moreover, we present a noise calibrated tensor power method with efficient privacy guarantees. At the heart of all these guarantees lies a  careful perturbation analysis derived in this paper which improves up on the existing results significantly.

\end{abstract}

\vspace{-0.1cm}
\section{Introduction}
\vspace{-0.1cm}

In recent years, tensor decomposition has emerged as a powerful tool to solve many challenging problems in unsupervised~\cite{anandkumar2014tensor}, supervised~\cite{janzamin2015generalization} and reinforcement learning~\cite{DBLP:journals/corr/Azizzadenesheli16}. Tensors are higher order extensions of matrices which can reveal far greater information compared to matrices, while retaining most of the efficiencies of matrix operations~\cite{anandkumar2014tensor}.

A central task in tensor analysis is the process of decomposing the tensor into  its rank-$1$ components,
which is usually referred to as \emph{CP (Candecomp/Parafac) decomposition} in the literature.
While decomposition of arbitrary tensors is NP-hard~\cite{hillar2013most}, it becomes tractable for the class of  tensors with linearly independent components. Through a simple {\em whitening} procedure, such tensors can be converted to orthogonally decomposable tensors.   Tensor power method is a popular method for computing the   decomposition of an orthogonal tensor. It is  simple and efficient to implement, and a natural extension of the matrix power method.

In the absence of noise, the tensor power method correctly recovers the components under a random initialization followed by deflation.
On the other hand, perturbation analysis of tensor power method   is much more delicate compared to the matrix case.  This is because the problem of tensor decomposition is NP-hard, and if a large amount of arbitrary noise is added to an orthogonal tensor, the decomposition can again become intractable. In~\cite{anandkumar2014tensor}, guaranteed recovery of components was proven under bounded noise, and the bound was improved in~\cite{anandkumar2015sample}. In this paper, we significantly improve upon the noise requirements, i.e.  the extent of noise that can be withstood by the tensor power method.

In order for tensor methods to be deployed in large-scale systems, we require fast, parallelizable and scalable algorithms. To achieve this, we need to avoid the exponential increase in computation and memory requirements with the order of the tensor; i.e. a naive implementation on a 3rd-order $d$-dimensional tensor would require $O(d^3)$ computation and memory. Instead, we analyze the online  tensor power method that requires only linear (in $d$) memory and does not form the entire tensor. This is achieved in settings, where the tensor is an empirical higher order moment, computed from the stream of data samples.
We can avoid explicit construction of the tensor by running online tensor power method directly on  i.i.d.~data samples.
We show that this algorithm correctly recovers tensor components in time\footnote{$\tilde{O}$ hides poly-logarithmic factors.} $\tilde{O}(n k^2 d)$ and $\tilde{O}(dk)$ memory for a rank-$k$ tensor and $n$ number of data samples. Additionally, we provide efficient sample complexity analysis.



As spectral methods become increasingly popular with recommendation system and health analytics applications \cite{wang2014spectral,huang2014scalable},
data privacy is particularly relevant in the context of preserving sensitive private information.
Differential privacy could still be useful even if data privacy is not the prime concern \cite{zemel2013learning}.
We propose the first differentially private tensor decomposition algorithm with both privacy and utility guarantees via noise calibrated power iterations.
We show that under the natural assumption of tensor incoherence,
privacy parameters have no (polynomial) dependence on tensor dimension $d$.
On the other hand, straightforward input perturbation type methods lead to far worse bounds and do not yield guaranteed recovery for all values of privacy parameters.
%

%
%
%
%

\subsection{Related work}

\paragraph{Online tensor SGD}
Stochastic gradient descent (SGD) is an intuitive approach for online tensor decomposition
and has been successful in practical large-scale tensor decomposition problems \cite{huang2015online}.
Despite its simplicity, theoretical properties are particularly hard to establish.
\cite{ge2015escaping} considered a variant of the SGD objective and proved its correctness.
However, the approach in \cite{ge2015escaping} only works for even-order tensors
and its sample complexity dependency upon tensor dimension $d$ is poor.

\paragraph{Tensor PCA}
In the \emph{statistical tensor PCA} \cite{montanari2014statistical} model a $d\times d\times d$ tensor $\mat T=\vct v^{\otimes 3}+\mat E$
is observed and one wishes to recover component $\vct v$ under the presence of Gaussian random noise $\mat E$.
\cite{montanari2014statistical} shows that $\|\mat E\|_\op=O(d^{-1/2})$ is sufficient to guarantee approximate recovery of $\vct v$
and \cite{hopkins2015tensor} further improves the noise condition to $\|\mat E\|_\op=O(d^{-1/4})$ via a 4th-order sum-of-squares relaxation.
Techniques in both \cite{montanari2014statistical,hopkins2015tensor} are rather complicated and could be difficult to adapt to memory or privacy constraints.
Furthermore, in \cite{montanari2014statistical,hopkins2015tensor} only one component is considered.
On the other hand, \cite{mu2015successive} shows that $\|\mat E\|_\op=O(d^{-1/2})$ is sufficient for recovering multiple components from noisy tensors.
However, \cite{mu2015successive} assumes exact computation of rank-1 tensor approximation, which is NP-hard in general.

\paragraph{Noisy matrix power methods}
Our relaxed noise condition analysis for tensor power method is inspired by recent analysis of noisy matrix power methods \cite{hardt2014noisy,balcan2016improved}.
Unlike the matrix case, tensor decomposition no longer requires \emph{spectral gap} among eigenvalues and eigenvectors are usually recovered one at a time \cite{anandkumar2014tensor,anandkumar2015sample}.
This poses new challenges and requires non-trivial extensions of matrix power method analysis to the tensor case.

\subsection{Notation and Preliminaries}
We use $[n]$ to denote the set $\{1,2,\cdots,n\}$.
We use bold characters $\mat A,\mat T,\vct v$ for matrices, tensors, vectors and normal characters $\lambda,\mu$ for scalars.
A $p$th order tensor $\mat T$ of dimensions $d_1,\cdots,d_p$ has $d_1\times\cdots\times d_p$ elements, each indexed by a $p$-tuple $(i_1,\cdots,i_p)\in[d_1]\times\cdots\times [d_p]$.
A tensor $\mat T$ of dimensions $d\times\cdots\times d$ is \emph{super-symmetric} or simply \emph{symmetric} if $\mat T_{i_1,\cdots,i_p} = \mat T_{\sigma(i_1),\cdots,\sigma(i_p)}$
for all permutations $\sigma:[p]\to[p]$.
For a tensor $\mat T\in\mathbb R^{d_1\times\cdots\times d_p}$ and matrices $\mat A_1\in\mathbb R^{m_1\times d_1}, \cdots, \mat A_p\in\mathbb R^{m_p\times d_p}$,
the \emph{multi-linear form} $\mat T(\mat A_1,\cdots,\mat A_p)$ is a $m_1\times\cdots\times m_p$ tensor defined as
$$
[\mat T(\mat A_1,\cdots,\mat A_p)]_{i_1,\cdots,i_p}
= \sum_{j_1\in[d_1]}\cdots\sum_{j_p\in[d_p]}{\mat T_{j_1,\cdots,j_p}[\mat A_1]_{j_1,i_1}\cdots[\mat A_p]_{j_p,i_p}}.
$$
We use $\|\vct v\|_2 = \sqrt{\sum_i{\vct v_i^2}}$ for vector 2-norm and $\|\vct v\|_{\infty} = \max_i{|\vct v_i|}$ for vector infinity norm.
We use $\|\mat T\|_\op$ to denote the \emph{operator norm} or \emph{spectral norm} of a tensor $\mat T$, which is defined as
$\|\mat T\|_\op = \sup_{\|\vct u_1\|_2=\cdots\|\vct u_p\|_2=1}\mat T(\vct u_1,\cdots,\vct u_p)$.
An event $\mathcal A$ is said to occur \emph{with overwhelming probability} if $\Pr[\mathcal A]\geq 1-d^{-10}$.

We limit ourselves to symmetric 3rd-order tensors ($p=3$) in this paper. The results can be directly extended to asymmetric tensors since they can first be symmetrized using simple matrix operations (see~\cite{anandkumar2014tensor}).
Extension to higher-order tensors is also straightforward.
A symmetric 3rd-order tensor $\mat T$ is rank-$1$ if it can be written in the form of
\begin{equation} \label{eqn:rank-1 tensor}
\mat T= \lambda \cdot \vct v \otimes \vct v\otimes \vct v = \lambda\vct v^{\otimes 3} \;\;\;\Longleftrightarrow \;\;\; \mat T_{i,j,\ell} = \lambda \cdot \vct v(i)\cdot\vct v(j)\cdot\vct v(\ell),
\end{equation}
where $\otimes$  represents the {\em outer product}, and $\vct v \in \Rbb^d$ is a unit vector (i.e., $\|\vct v\|_2=1$) and $\lambda\in \Rbb^+$.
\footnote{One can always assume without loss of generality that $\lambda \geq 0$ by replacing $\vct v$ with $-\vct v$ instead.}
A tensor $\mat T  \in \Rbb^{d \times d \times d}$ is said to have a CP (Candecomp/Parafac) {\em rank} $k$ if it can be (minimally) written as the sum of $k$ rank-$1$ tensors:
\begin{equation}\label{eqn:tensordecomp}
\mat T = \sum_{i\in [k]} \lambda_i\vct v_i \otimes\vct v_i \otimes\vct v_i, \quad \lambda_i \in \Rbb^+, \;\;\vct v_i \in \Rbb^d.
\end{equation}

A tensor is said to be orthogonally decomposable if in the above decomposition $\inner{\vct v_i, \vct v_j}=0$ for $i \neq j$. Any tensor can be converted to an orthogonal tensor through an invertible whitening transform,
provided that $\vct v_1 ,\vct v_2, \ldots , \vct v_k$ are linearly independent~\cite{anandkumar2014tensor}.
We thus limit our analysis to orthogonal tensors in this paper since it can be extended to this more general class in a straightforward manner.

\begin{algorithm}[t]
\centering
\caption{Robust tensor power method \cite{anandkumar2014tensor}}
\begin{algorithmic}[1]
\State \textbf{Input}: symmetric $d\times d\times d$ tensor $\widetilde{\mat T}$, number of components $k\leq d$, number of iterations $L$, $R$.
\For{$i=1$ to $k$}
\State \underline{\em Initialization}: Draw $\vct u_0$ uniformly at random from the unit sphere in $\mathbb R^d$.
\State \underline{\em Power iteration}: Compute $\vct u_t = \widetilde{\mat T}(\mat I,\vct u_{t-1},\vct u_{t-1})/\|\widetilde{\mat T}(\mat I,\vct u_{t-1},\vct u_{t-1})\|_2$ for $t=1,\cdots,R$.
\State \parbox[t]{\dimexpr\linewidth-\algorithmicindent}{\underline{\em Boosting}: Repeat Steps 3 and 4 for $L$ times and obtain $\vct u_R^{(1)},\cdots,\vct u_R^{(L)}$.
Let $\tau^* = \argmax_{\tau=1}^L{\widetilde{\mat T}(\vct u_R^{(\tau)},\vct u_R^{(\tau)},\vct u_R^{(\tau)})}$.
Set $\hat{\vct v}_i = \vct u_R^{(\tau)}$ and $\hat\lambda_i = \widetilde{\mat T}(\vct u_R^{(\tau)},\vct u_R^{(\tau)},\vct u_R^{(\tau)})$.\strut}
\State \underline{\em Deflation}: $\widetilde{\mat T} \gets \widetilde{\mat T} - \hat\lambda_i\hat{\vct v}_i^{\otimes 3}$.
\EndFor
\State \textbf{Output}: Estimated eigenvalue/Eigenvector pairs $\{\hat\lambda_i,\hat{\vct v}_i\}_{i=1}^k$.
\end{algorithmic}
\label{alg:rtpm}
\end{algorithm}
\vspace{-0.2cm}

\paragraph{Tensor Power Method: }A popular algorithm for finding the tensor decomposition in \eqref{eqn:tensordecomp} is through the tensor power method. The full algorithm is given  in Algorithm \ref{alg:rtpm}.
We first provide an improved noise analysis for the robust power method, improving error tolerance bounds previously established in \cite{anandkumar2014tensor}.
We next propose memory-efficient and/or differentially private variants of the robust power method
and give performance guarantee based on our improved noise analysis.

\vspace{-0.1cm}
\section{Improved Noise Analysis for Tensor Power Method}\label{sec:main}
\vspace{-0.1cm}


When the tensor $\mat T$ has an exact orthogonal decomposition, the power method provably recovers all the components with random initialization and deflation. However, the analysis is more subtle under noise. While matrix perturbation bounds are well understood, it is an open problem in the case of tensors. This is because the problem of tensor decomposition is NP-hard, and becomes tractable only under special conditions such as orthogonality (and more generally linear independence). If a large amount of arbitrary noise is added, the decomposition can again become intractable. In~\cite{anandkumar2014tensor}, guaranteed recovery of components was proven under bounded noise and we recap the result below.

\begin{thm}[\cite{anandkumar2014tensor} Theorem 5.1, simplified version]
Suppose $\widetilde{\mat T}=\mat T+\mat\Delta_T$, where $\mat T=\sum_{i=1}^k{\lambda_i\vct v_i^{\otimes 3}}$ with $\lambda_i>0$ and orthonormal basis vectors$\{\vct v_1,\cdots,\vct v_k\}\subseteq\mathbb R^d$, $d\geq k$,
and noise $\mat\Delta_T$ satisfies $\|\mat\Delta_T\|_\op\leq\epsilon$.
Let $\lambda_{\max},\lambda_{\min}$ be the largest and smallest values in $\{\lambda_i\}_{i=1}^k$ and $\{\hat\lambda_i,\hat{\vct v}_i\}_{i=1}^k$ be outputs of Algorithm \ref{alg:rtpm}.
There exist absolute constants $K_0,C_1,C_2,C_3>0$ such that if 
\begin{equation}
\epsilon\leq C_1\cdot \lambda_{\min}/d, \;\;
R = \Omega(\log d + \log\log(\lambda_{\max}/\epsilon)), \;\;
L = \Omega(\max\{K_0,k\}\log (\max\{K_0,k\})),
\label{eq_epsTL}
\end{equation}
then with probability at least $0.9$, there exists a permutation $\pi:[k]\to [k]$ such that
$$
|\lambda_i-\hat\lambda_{\pi(i)}| \leq C_2\epsilon, \;\;\;\; \|{\vct v}_i-\hat{\vct v}_{\pi(i)}\|_2\leq C_3\epsilon/\lambda_{i},\;\;\;\;
\forall i=1,\cdots,k.
$$
\label{thm:rtpm}
\end{thm}
\vspace{-0.2cm}

Theorem \ref{thm:rtpm} is the first provably correct result  on robust tensor decomposition under general noise conditions.
In particular, the noise term $\mat\Delta_T$ can be deterministic or even adversarial.
However, one important drawback of Theorem \ref{thm:rtpm} is that $\|\mat\Delta_T\|_\op$ must be upper bounded by $O(\lambda_{\min}/d)$,
which is a strong assumption for many practical applications \cite{wang2015fast}.
On the other hand, \cite{anandkumar2015sample,montanari2014statistical} show that by using smart initializations the robust tensor power method is capable
of tolerating $O(\lambda_{\min}/\sqrt{d})$ magnitude of noise,
and \cite{mu2015successive} suggests that such noise magnitude cannot  be improved if deflation (i.e., successive rank-one approximation) is
to be performed.

In this paper, we show that the relaxed noise bound $O(\lambda_{\min}/\sqrt{d})$ holds even if the initialization of robust TPM is as simple as
a vector uniformly sampled from the $d$-dimensional sphere (Algorithm \ref{alg:rtpm}).
Our claim is formalized below:
\begin{thm}[Improved noise tolerance analysis for robust TPM]
Assume the same notation as in Theorem \ref{thm:rtpm}.
Let $\epsilon\in(0,1/2)$ be an error tolerance parameter.
There exist absolute constants $K_0,C_0,C_1,C_2,C_3>0$ such that if $\mat\Delta_T$ satisfies
\begin{equation}
\|\mat\Delta_T(\mat I,\vct u_t^{(\tau)},\vct u_t^{(\tau)})\|_2 \leq \epsilon, \;\;\;\;
|\mat\Delta_T(\vct v_i,\vct u_t^{(\tau)}, \vct u_t^{(\tau)})| \leq \min\{\epsilon/\sqrt{k}, C_0\lambda_{\min}/d\}
\label{eq_deltaT}
\end{equation}
for all $i\in[k]$, $t\in[T],\tau\in[L]$ and furthermore
\begin{equation}
\epsilon\leq C_1\cdot \lambda_{\min}/\sqrt{k},\;\;\;\;
R=\Omega(\log (\lambda_{\max}d/\epsilon)), \;\;\;\; L = \Omega(\max\{K_0,k\}\log (\max\{K_0,k\})),
\label{eq_epsTL_new}
\end{equation}
then with probability at least $0.9$, there exists a permutation $\pi:[k]\to[k]$ such that
$$
|\lambda_i-\hat\lambda_{\pi(i)}| \leq C_2\epsilon, \;\;\;\; \|{\vct v}_i-\hat{\vct v}_{\pi(i)}\|_2\leq C_3\epsilon/\lambda_{i},\;\;\;\;
\forall i=1,\cdots,k.
$$
\label{thm:rtpm-new}
\end{thm}
\vspace{-0.2cm}
Due to space constraints, proof of Theorem \ref{thm:rtpm-new} is placed in Appendix \ref{appsec:main-proof-rtpm}.
We next make several remarks on our results.
In particular, we consider three scenarios with increasing assumptions imposed on the noise tensor $\mat\Delta_T$
and compare the noise conditions in Theorem \ref{thm:rtpm-new} with existing results on orthogonal tensor decomposition:
\begin{enumerate}
\item \emph{$\mat\Delta_T$ does not have any special structure}: in this case, we only have $|\mat\Delta_T(\vct v_i,\vct u_t,\vct u_t)| \leq\|\mat\Delta_T\|_{\op}$
and our noise conditions reduces to the classical one: $\|\mat\Delta_T\|_\op = O(\lambda_{\min}/d)$.
\item \emph{$\mat\Delta_T$ is ``round''} in the sense that $|\mat\Delta_T(\vct v_i,\vct u_t,\vct u_t)| \leq O(1/\sqrt{d})\cdot \|\mat\Delta_T(\mat I,\vct u_t,\vct u_t)\|_2$:
this is the typical setting when the noise $\mat\Delta_T$ follows Gaussian or sub-Gaussian distributions, as we explain in Sec.~\ref{sec:streaming} and \ref{sec:private}.
Our noise condition in this case is $\|\mat\Delta_T\|_\op = O(\lambda_{\min}/\sqrt{d})$, strictly improving Theorem \ref{thm:rtpm} on robust tensor power method
with random initializations and matching the bound for more advanced SVD initialization techniques in \cite{anandkumar2015sample}.
\item \emph{$\mat\Delta_T$ is weakly correlated with signal} in the sense that $\|\mat\Delta_T(\vct v_i,\mat I,\mat I)\|_2=O(\lambda_{\min}/d)$ for all $i\leq k$:
in this case our noise condition reduces to $\|\mat\Delta_T\|_{\op} = O(\lambda_{\min}/\sqrt{k})$, strictly improving over SVD initialization \cite{anandkumar2015sample}
in the ``undercomplete'' regime $k=o(d)$.
Note that the whitening trick \cite{anandkumar2012spectral,anandkumar2014tensor} does not attain our bound, as we explain in Appendix \ref{appsec:whitening}.
\end{enumerate}
Finally, we remark that the $\log\log(1/\epsilon)$ quadratic convergence rate in Eq.~(\ref{eq_epsTL}) is worsened to $\log(1/\epsilon)$ linear rate in Eq.~(\ref{eq_epsTL_new}).
We are not sure whether this is an artifact of our analysis, because similar analysis for the matrix noisy power method \cite{hardt2014noisy}
also reveals a linear convergence rate.

\paragraph{Implications}
Our bounds in Theorem \ref{thm:rtpm-new} results in sharper analysis of both memory-efficient and differentially private power methods which we propose in Sec.~\ref{sec:streaming}, \ref{sec:private}.
Using the original analysis (Theorem \ref{thm:rtpm}) for the two applications,
the memory-efficient tensor power method would have sample complexity \emph{cubic} in the dimension $d$ and
for differentially private tensor decomposition the privacy level $\varepsilon$ needs to scale as $\tilde\Omega(\sqrt{d})$ as $d$ increases,
which is particularly bad as the quality of privacy protection $e^{\varepsilon}$ degrades exponentially with tensor dimension $d$.
On the other hand, our improved noise condition in Theorem \ref{thm:rtpm-new} greatly sharpens the bounds in both applications: for memory efficient decomposition, we now require only quadratic sample complexity and for differentially private decomposition, the privacy level $\varepsilon$ has no polynomial dependence on $d$. This makes our results far   more practical for high-dimensional tensor decomposition applications.

\paragraph{Numerical verification of noise conditions and comparison with whitening techniques}
We verify our improved noise conditions for robust tensor power method on simulation tensor data.
In particular, we consider three noise models and demonstrate varied asymptotic noise magnitudes at which tensor power method succeeds.
The simulation results nicely match our theoretical findings and also suggest, in an empirical way, tightness of noise bounds in Theorem \ref{thm:rtpm-new}.
Due to space constraints, simulation results are placed in Appendix \ref{appsec:simulation}.

We also compare our improved noise bound with those obtained by \emph{whitening}, 
a popular technique that reduces tensor decomposition to matrix decomposition problems \cite{anandkumar2014tensor,kuleshov2015tensor,wang2015fast}.
We show in Appendix \ref{appsec:whitening} that, without side information the standard analysis of whitening based tensor decomposition
leads to worse noise tolerance bounds than what we obtained in Theorem \ref{thm:rtpm-new}.

\vspace{-0.1cm}
\section{Memory-Efficient Streaming Tensor Decomposition}\label{sec:streaming}
\vspace{-0.1cm}

Tensor power method in Algorithm \ref{alg:rtpm} requires significant storage to be deployed: $\Omega(d^3)$ memory is required to store a dense $d\times d\times d$ tensor,
which is prohibitively large in many real-world applications as tensor dimension $d$ could be really high.
We show in this section how to compute tensor decomposition in a memory efficient manner,
with storage scaling \emph{linearly} in $d$.
In particular, we consider the case when tensor $\mat T$ to be decomposed is a \emph{population moment} $\mathbb E_{\vct x\sim\mathcal D}[\vct x^{\otimes 3}]$
with respect to some unknown underlying data distribution $\mathcal D$,
and data points $\vct x_1,\vct x_2,\cdots$ i.i.d.~sampled from $\mathcal D$ are fed into a tensor decomposition algorithm in a streaming fashion.
One classical example is topic modeling, where each $\vct x_i$ represents documents that come in streams and consistent estimation of topics can be achieved by
decomposing variants of the population moment \cite{anandkumar2014tensor,anandkumar2012spectral}.

Algorithm \ref{alg:streaming-rtpm} displays memory-efficient tensor decomposition procedure on streaming data points.
The main idea is to replace the power iteration step $\mat T(\mat I,\vct u,\vct u)$ in Algorithm \ref{alg:rtpm} with a ``data association'' step
that exploits the empirical-moment structure of the tensor $\mat T$ to be decomposed and evaluates approximate power iterations from stochastic data samples.
This procedure is highly efficient, in that both time and space complexity scale linearly with tensor dimension $d$:
\begin{prop}
Algorithm \ref{alg:streaming-rtpm} runs in $O(nkdLR)$ time and $O(d(k+L))$ memory, with $O(nkR)$ sample complexity (number of data point gone through).
\label{prop_complexity}
\end{prop}

In the remainder of this section we show Algorithm \ref{alg:streaming-rtpm} recovers eigenvectors of the population moment $\mathbb E_{\vct x\sim\mathcal D}[\vct x^{\otimes 3}]$
with high probability and we derive corresponding sample complexity bounds.
To facilitate our theoretical analysis we need several assumptions on the data distribution $\mathcal D$.
The first natural assumption is the low-rankness of the population moment $\mathbb E_{\vct x\sim\mathcal D}[\vct x^{\otimes 3}]$ to be decomposed:
\begin{asmp}[Low-rank moment]
The mean tensor $\mat T=\mathbb E_{\vct x\sim\mathcal D}[\vct x^{\otimes 3}]$
admits a low-rank representation $\mat T=\sum_{i=1}^k{\lambda_i\vct v_i^{\otimes 3}}$
for $\lambda_1,\cdots,\lambda_k > 0$ and orthonormal $\{\vct v_1,\cdots,\vct v_k\}\subseteq\mathbb R^d$.
\label{asmp_mean}
\end{asmp}

We also place restrictions on the ``noise model'', which imply that the population moment $\mathbb E_{\vct x\sim\mathcal D}[\vct x^{\otimes 3}]$
can be well approximated by a reasonable number of samples with high probability.
In particular, we consider sub-Gaussian noise as formulated in Definition \ref{defn_subgaussian} and Assumption \ref{asmp_subgaussian}:
\begin{defn}[Multivariate sub-Gaussian distribution, \cite{hsu2012tail}]
A $D$-dimensional random variable $\vct x$ belongs to the sub-Gaussian distribution family $\sg_D(\sigma)$ with parameter $\sigma>0$
if it has zero mean and
$
\mathbb E\left[\exp(\vct a^\top\vct x)\right]  \leq \exp\left\{\|\vct a\|_2^2\sigma^2/2\right\}
$
for all $\vct a\in\mathbb R^D$.
\label{defn_subgaussian}
\end{defn}
\begin{asmp}[Sub-Gaussian noise]
There exists $\sigma>0$ such that
the mean-centered vectorized random variable $\VEC(\vct x^{\otimes 3}-\mathbb E[\vct x^{\otimes 3}])$ belongs to
$\sg_{d^3}(\sigma)$ as defined in Definition \ref{defn_subgaussian}.
\label{asmp_subgaussian}
\end{asmp}
We remark that Assumption \ref{asmp_subgaussian} includes a wide family of distributions that are of practical importance,
for example noise that have compact support.
Assumption \ref{asmp_subgaussian} also resembles \emph{$(B,p)$-round noise} considered in \cite{hardt2014noisy}
that imposes spherical symmetry constraints onto the noise distribution.

\begin{algorithm}[t]
\caption{Online robust tensor power method}
\begin{algorithmic}[1]
\State \textbf{Input}: data stream $\vct x_1,\vct x_2,\cdots\in\mathbb R^d$, no.~of components $k$, parameters $L,R,n$.
\For{$i=1$ to $k$}
	\State Draw $\vct u_0^{(1)},\cdots,\vct u_0^{(L)}$ i.i.d.~uniformly at random from the unit sphere $\mathcal S^{d-1}$.
	\For{$t=0$ to $R-1$}
		\State \underline{Initialization}: Set accumulators $\tilde{\vct u}_{t+1}^{(1)}, \cdots, \tilde{\vct u}_{t+1}^{(L)}$ and $\tilde\lambda^{(1)},\cdots,\tilde\lambda^{(L)}$ to 0.
		\State \parbox[t]{\dimexpr\linewidth-\algorithmicindent}{\underline{Data association}: Read the next $n$ data points;
		update $\tilde{\vct u}_{t+1}^{(\tau)} \gets \tilde{\vct u}_{t+1}^{(\tau)} + \frac{1}{n}(\vct x_\ell^\top\vct u_t^{(\tau)})^{2}\vct x_i$\\ and
		$\tilde\lambda^{(\tau)} \gets \tilde\lambda^{(\tau)} + \frac{1}{n}(\vct x_\ell^\top\vct u_t^{(\tau)})^3$
		 for each $\ell\in[n]$ and $\tau\in[L]$.\strut}
		\State \parbox[t]{\dimexpr\linewidth-\algorithmicindent}{\underline{Deflation}: For each $\tau\in[L]$, update
		$\tilde{\vct u}_{t+1}^{(\tau)} \gets \tilde{\vct u}_{t+1}^{(\tau)} - \sum_{j=1}^{i-1}{\hat\lambda_j\xi_{j,\tau}^{2}\hat{\vct v}_j}$\\ and
		$\tilde\lambda^{(\tau)}\gets \tilde\lambda^{(\tau)}-\sum_{j=1}^{i-1}{\hat\lambda_j\xi_{j,\tau}^3}$, where $\xi_{j,\tau}=\hat{\vct v}_j^\top\tilde{\vct u}_t^{(\tau)}$.\strut}
		\State \underline{Normalization}: $\vct u_{t+1}^{(\tau)} = \tilde{\vct u}_{t+1}^{(\tau)}/\|\tilde{\vct u}_{t+1}^{(\tau)}\|_2$, for each $\tau\in[L]$.
	\EndFor
	\State Find $\tau^* = \argmax_{\tau\in[L]}\tilde\lambda^{(\tau)}$ and store $\hat\lambda_i = \tilde\lambda^{(\tau^*)}$, $\hat{\vct v}_i = \vct u_R^{(\tau^*)}$.
\EndFor
\State \textbf{Output}: approximate eigenvalue and eigenvector pairs $\{\hat\lambda_i,\hat{\vct v}_i\}_{i=1}^k$ of $\hat{\mathbb E}_{\vct x\sim\mathcal D}[\vct x^{\otimes 3}]$.
\end{algorithmic}
\label{alg:streaming-rtpm}
\end{algorithm}

We are now ready to present the main theorem that bounds the recovery (approximation) error of eigenvalues and eigenvectors
of the streaming robust tensor power method in Algorithm \ref{alg:streaming-rtpm}:
\begin{thm}[Analysis of streaming robust tensor power method]
Let Assumptions \ref{asmp_mean}, \ref{asmp_subgaussian} hold true and suppose $\epsilon < C_1\lambda_{\min}/\sqrt{k}$ for some sufficiently small absolute constant $C_1>0$.
If
$$
n = \widetilde\Omega\left(\min\left\{\frac{\sigma^2d}{\epsilon^2}, \frac{\sigma^2d^2}{\lambda_{\min}^2}\right\}\right), \;\;\;\;
R = \Omega(\log(\lambda_{\max}d/\epsilon)), \;\;\;\;
L = \Omega(k\log k),
$$
then with probability at least 0.9 there exists permutation $\pi:[k]\to [k]$ such that
$$
|\lambda_i-\hat\lambda_{\pi(i)}| \leq C_2\epsilon, \;\;\;\; \|{\vct v}_i-\hat{\vct v}_{\pi(i)}\|_2\leq C_3\epsilon/\lambda_{i},\;\;\;\;
\forall i=1,\cdots,k
$$
for some universal constants $C_2,C_3>0$.
\label{thm:streaming}
\end{thm}
Corollary \ref{cor_streaming} is then an immediate consequence of Theorem \ref{thm:streaming},
which simplifies the bounds and highlights asymptotic dependencies over important model parameters $d,k$ and $\sigma$:
\begin{cor}
Under Assumptions \ref{asmp_mean}, \ref{asmp_subgaussian},
Algorithm \ref{alg:streaming-rtpm} correctly learns $\{\lambda_i,\vct v_i\}_{i=1}^k$ up to $O(1/\sqrt{d})$ additive error
with $\tilde O(\sigma^2 kd^2)$ samples and $\tilde O(dk)$ memory.
\label{cor_streaming}
\end{cor}

Proofs of Theorem \ref{thm:streaming} and Corollary \ref{cor_streaming} are both deferred to Appendix \ref{appsec:proof-streaming}.
Compared to streaming noisy matrix PCA considered in \cite{hardt2014noisy}, the bound is weaker with an additional
$1/k$ factor in the term involving $\epsilon$ and $1/d$ factor in the term that does not involve $\epsilon$.
We conjecture this to be a fundamental difficulty of the tensor decomposition problem. 
On the other hand, our bounds resulting from the analysis in Sec.~\ref{sec:main} have a $O(1/d)$ improvement compared to applying existing analysis in \cite{anandkumar2014tensor} directly.

\paragraph{Remark on comparison with SGD: }
Our proposed streaming tensor power method is nothing but  the projected   stochastic gradient descent (SGD) procedure on the objective of maximizing  the  tensor norm on the sphere. The optimal solution of this coincides with the objective of finding the best rank-$1$ approximation of the tensor. Here, we can estimate   all the components of the tensor through deflation.
An alternative method is to run SGD based a combined objective function to obtain all the components of the tensor simultaneously, as considered in \cite{huang2015online,ge2015escaping}.
However, the analysis in \cite{ge2015escaping} only works for even-order tensors and has worse dependency (at least $d^9$) on tensor dimension $d$.

\vspace{-0.1cm}
\section{Differentially private tensor decomposition}\label{sec:private}
\vspace{-0.1cm}

The objective of private data processing is to  release data summaries such that any particular entry of the original data cannot be reliably inferred from the released results.
Formally speaking, we adopt the popular $(\varepsilon,\delta)$-differential privacy criterion proposed in \cite{dwork2014algorithmic}:
\begin{defn}[$(\varepsilon,\delta)$-differential privacy \cite{dwork2014algorithmic}]\label{def-private}
Let $\mathcal M$ denote all symmetric $d$-dimensional real third order tensors and $\mathcal O$ be an arbitrary output set.
A randomized algorithm $A:\mathcal M\to\mathcal O$ is \emph{$(\varepsilon,\delta)$-differentially private} if for all neighboring tensors $\mat T,\mat T'$
and measurable set $O\subseteq\mathcal O$ we have
$$
\Pr\left[A(\mat T)\in O\right] \leq e^{\varepsilon}\Pr\left[A(\mat T')\in O\right] + \delta,
$$
where $\varepsilon >0$, $\delta\in[0,1)$ are privacy parameters and probabilities are taken over randomness in $A$.
\end{defn}
Since our tensor decomposition analysis concerns symmetric tensors primarily,
we adopt a ``symmetric'' definition of neighboring tensors in Definition \ref{def-private}, as shown below:
\begin{defn}[Neighboring tensors] 
Two $d\times d\times d$ symmetric tensors $\mat T,\mat T'$ are \emph{neighboring tensors} if there exists $i,j,k\in[d]$ such that
$$
\mat T' - \mat T = \pm\mathrm{symmetrize}(\vct e_i\otimes\vct e_j\otimes\vct e_k) = \pm\left(\vct e_i\otimes\vct e_j\otimes\vct e_k + \vct e_i\otimes\vct e_k\otimes\vct e_j + \cdots + \vct e_k\otimes\vct e_j\otimes\vct e_i\right).
\vspace{-0.1cm}
$$
\label{def-neighbor}
\end{defn} 
As noted earlier, the above notions can be similarly extended to asymmetric tensors as well as the guarantees for tensor power method on asymmetric tensors.
We also remark that the difference of ``neighboring tensors'' as defined above has Frobenious norm bounded by $O(1)$.
This is necessary because an arbitrary perturbation of a tensor, even if restricted to only one entry,
is capable of destroying any utility guarantee possible.

In a nutshell, Definitions~\ref{def-private}, \ref{def-neighbor} state that an algorithm $A$ is differentially private if, conditioned on any set of possible outputs of $A$, one cannot distinguish with high probability between two ``neighboring'' tensors $\mat T,\mat T'$
that differ only in a single entry (up to symmetrization),
thus protecting the privacy of any particular element in the original tensor $\mat T$.
Here $\varepsilon,\delta$ are parameters controlling the level of privacy, with smaller $\varepsilon,\delta$ values implying stronger privacy guarantee
as $\Pr[A(\mat T)\in O]$ and $\Pr[A(\mat T')\in O]$ are closer to each other.

\begin{algorithm}[t]
\caption{Differentially private robust tensor power method}
\begin{algorithmic}[1]
\State \textbf{Input}:  tensor $\mat T$, no. of components $k$, number of iterations $L,R$, privacy parameters $\varepsilon,\delta$.
\State \textbf{Initialization}: $\mat D=\mat 0$, $\nu = \frac{6\sqrt{2\ln({1.25}/{\delta'})}}{\varepsilon'}$, $\delta'=\frac{\delta}{2K}$, $\varepsilon'=\frac{\varepsilon}{\sqrt{K(4+\ln({2}/{\delta}))}}$,
$K=kL(R+1)$.
\For{$i=1$ to $k$}
\State \underline{\em Initialization}: Draw $\vct u_0^{(1)},\cdots,\vct u_0^{(\tau)}$ uniformly at random from the unit sphere in $\mathbb R^d$.
\For{$t=0$ to $R-1$}
\State \underline{\em Power iteration}: compute $\tilde{\vct u}_{t+1}^{(\tau)} = (\mat T-\mat D)(\mat I, \vct u_t^{(\tau)},\vct u_t^{(\tau)})$.
\State \underline{\em Noise calibration}: release $\bar{\vct u}_{t+1}^{(\tau)} = \tilde{\vct u}_{t+1}^{(\tau)} + \nu\|\vct u_t^{(\tau)}\|_\infty^2\cdot \vct z_t^{(\tau)}$, where
$\vct z_t^{(\tau)}\overset{i.i.d.}{\sim} \mathcal N(\vct 0,\mat I_d)$.
\State \underline{\em Normalization}: $\vct u_{t+1}^{(\tau)} = \bar{\vct u}_{t+1}^{(\tau)} / \|\bar{\vct u}_{t+1}^{(\tau)}\|_2$.
\EndFor
\State Compute $\tilde \lambda^{(\tau)}= (\mat T-\mat D)(\vct u_R^{(\tau)},\vct u_R^{(\tau)},\vct u_R^{(\tau)}) + \nu\|\vct u_R^{(\tau)}\|_{\infty}^3\cdot z_\tau$
and let $\tau^* = \argmax_\tau \tilde\lambda^{(\tau)}$.
\State \underline{\em Deflation}: $\hat\lambda_i = \tilde\lambda^{(\tau^*)}$, $\hat{\vct v}_i = \vct u_R^{(\tau^*)}$, $\mat D\gets \mat D + \hat\lambda_i\hat{\vct v}_i^{\otimes 3}$.
\EndFor
\State \textbf{Output}: eigenvalue/eigenvector pairs $\{\hat\lambda_i,\hat{\vct v}_i\}_{i=1}^k$.
\end{algorithmic}
\label{alg_private_rtpm}
\end{algorithm}

Algorithm \ref{alg_private_rtpm} describes the procedure of privately releasing eigenvectors of a low-rank input tensor $\mat T$. 
The main idea for privacy preservation is the following \emph{noise calibration} step
$$
\bar{\vct u}_{t+1} = \tilde{\vct u}_{t+1} + \nu\|\vct u_t\|_\infty^2\cdot\vct z_t,
$$
where $\vct z_t$ is a $d$-dimensional standard Normal random variable and $\nu\|\vct u_t\|_\infty^2$ is a carefully designed noise magnitude in order to achieved desired privacy level $(\varepsilon,\delta)$.
One key aspect is that the noise calibration step occurs at \emph{every} power iteration,
which adds to the robustness of the algorithm and achieves sharper bounds.
We discuss at the end of this section.

\begin{thm}[Privacy guarantee]
Algorithm \ref{alg_private_rtpm} satisfies $(\varepsilon,\delta)$-differential privacy.
\label{prop_privacy}
\end{thm}
\begin{proof}
The only power iteration step of Algorithm \ref{alg_private_rtpm} can be thought of as $K=kL(R+1)$ queries directed to a private data sanitizer
which produces $f_1(\mat T;\vct u)=\mat T(\mat I,\vct u,\vct u)$ or $f_2(\mat T;\vct u)=\mat T(\vct u,\vct u,\vct u)$ each time.
The $\ell_2$-sensitivity of both queries can be separately bounded as
\begin{eqnarray*}
\Delta_2 f_1 &=& \sup_{\mat T'}\left\| \mat T(\mat I,\vct u,\vct u)-\mat T'(\mat I,\vct u,\vct u)\right\|_2
\leq \sup_{i,j,k} 2(|\vct u_i\vct u_j| + |\vct u_i\vct u_k|+|\vct u_j\vct u_k|) \leq 6\|\vct u\|_\infty^2;\\
\Delta_2 f_2 &=& \sup_{\mat T'}\big|\mat T(\vct u,\vct u,\vct u)-\mat T'(\vct u,\vct u,\vct u)\big|
= \sup_{i,j,k} 6\big|\vct u_i\vct u_j\vct u_k\big| \leq 6\|\vct u\|_\infty^3,
\end{eqnarray*}
where $\mat T'=\mat T+\mathrm{symmetrize}(\vct e_i\otimes\vct e_j\otimes\vct e_k)$ is some neighboring tensor of $\mat T$.
Thus, applying the Gaussian mechanism \cite{dwork2014algorithmic} we can $(\varepsilon,\delta)$-privately release \emph{one output} of either $f_1(\vct u)$ or $f_2(\vct u)$ by
$$
f_\ell(\vct u) + \frac{\Delta_2f_\ell\cdot \sqrt{2\ln(1.25/\delta)}}{\varepsilon}\cdot \vct w,
$$
where $\ell=1,2$ and $\vct w\sim\nml(\vct 0,\mat I)$ are i.i.d.~standard Normal random variables.
Finally, applying \emph{advanced composition} \cite{dwork2014algorithmic} across all $K=kL(R+1)$ private releases we complete the proof of this proposition.
Note that both normalization and deflation steps do not affect the differential privacy of Algorithm \ref{alg_private_rtpm}
due to the \emph{closeness under post-processing} property of DP.
\end{proof}

The rest of the section is devoted to discussing the ``utility'' of Algorithm \ref{alg_private_rtpm}; i.e., to show that the algorithm is still capable of producing approximate eigenvectors,
despite the privacy constraints.
Similar to \cite{hardt2014noisy}, we adopt the following incoherence assumptions on the eigenspace of $\mat T$:
\begin{asmp}[Incoherent basis]
Suppose $\mat V\in\mathbb R^{d\times k}$ is the stacked matrix of orthonormal component vectors $\{\vct v_i\}_{i=1}^k$.
There exists constant $\mu_0>0$ such that
\begin{equation}
\vspace{-0.2cm}
\frac{d}{k}\max_{1\leq i\leq d}\|\mat V^\top\vct e_i\|_2^2 \leq \mu_0.
\label{eq_coherence}
\vspace{-0.2cm}
\end{equation}
\label{asmp_coherence}
\end{asmp}
Note that by definition, $\mu_0$ is always in the range of $[1,d/k]$.
Intuitively, Assumption \ref{asmp_coherence} with small constant $\mu_0$ implies a relatively ``flat'' distribution of element magnitudes in $\mat T$.
The incoherence level $\mu_0$ plays an important role in the utility guarantee of Algorithm \ref{alg_private_rtpm}, as we show below:
\begin{thm}[Guaranteed recovery of   eigenvector under privacy requirements]
Suppose $\mat T=\sum_{i=1}^k{\lambda_i\vct v_i^{\otimes 3}}$ for $\lambda_1>\lambda_2\geq \lambda_3\geq\cdots\geq\lambda_k > 0$ with orthonormal $\vct v_1,\cdots,\vct v_k\in\mathbb R^d$,
and suppose Assumption \ref{asmp_coherence} holds with $\mu_0$.
Assume $\lambda_1-\lambda_2 \geq c/\sqrt{d}$ for some sufficiently small universal constant $c>0$.
If $R=\Theta(\log(\lambda_{\max} d))$, $L=\Theta(k\log k)$ and $\varepsilon,\delta$ satisfy
\begin{equation}
\varepsilon = \Omega\left( \frac{\mu_0 k^2\log(\lambda_{\max} d/\delta)}{\lambda_{\min}}\right),
\end{equation}
then with probability at least 0.9 the first eigen pair $(\hat{\lambda}_1,\hat{\vct v}_1)$
returned by Algorithm \ref{alg_private_rtpm} satisfies
$$
\big|\lambda_1-\hat{\lambda_1}\big|  = O(1/\sqrt{d}), \;\;\;\;\;\;
\|\vct v_1-\hat{\vct v}_1\|_2 = O(1/(\lambda_1\sqrt{d})).
$$
\label{thm_private}
\end{thm}
\vspace{-0.2cm}
At a high level, Theorem \ref{thm_private} states that when the privacy parameter $\varepsilon$ is not too small (i.e., privacy requirements are not too stringent),
Algorithm \ref{alg_private_rtpm} approximately recovers the largest eigenvalue and eigenvector
with high probability.
Furthermore,
when $\mu_0$ is a constant, the lower bound condition on the privacy parameter $\varepsilon$ does \emph{not} depend polynomially upon tensor dimension $d$,
which is a much desired property for high-dimensional data analysis.
On the other hand, similar results cannot be achieved via simpler methods like input perturbation, as we discuss below:

\paragraph{Comparison with input perturbation}
Input perturbation is perhaps the simplest method for differentially private data analysis and has been  successful in numerous scenarios, e.g. private matrix PCA  \cite{dwork2014analyze}.
In our context, this would entail   appending a random Gaussian tensor $\mat E$ directly onto the input tensor $\mat T$ \emph{before} tensor power iterations.
By Gaussian mechanism, the standard deviation $\sigma$ of each element in $\mat E$  scales as $\sigma=\Omega(\varepsilon^{-1}\sqrt{\log(1/\delta)})$. On the other hand, noise analysis for tensor decomposition derived in  \cite{montanari2014statistical,anandkumar2015sample} and in the subsequent section of this paper
requires $\sigma=O(1/d)$ or $\|\mat E\|_\op=O(1/\sqrt{d})$, which implies $\varepsilon = \tilde\Omega(d)$ (cf.~Lemma \ref{lem_tensor_spectral_norm}).
That is, the privacy parameter $\varepsilon$ must scale \emph{linearly} with tensor dimension $d$ to successfully recover even the first principle eigenvector,
which renders the privacy guarantee of the input perturbation procedure useless for high-dimensional tensors.
Thus, we require a non-trivial new approach for differentially private tensor decomposition.


Finally, we remark that a more desired utility analysis would bound the approximation error $\|\vct v_i-\hat{\vct v}_i\|_2$ for every component $\vct v_1,\cdots,\vct v_k$, and not just the top eigenvector.
Unfortunately, our current analysis cannot handle deflation effectively as the deflated vector $\hat{\vct v}_i-\vct v_i$ may not be incoherent.
Extension to deflated tensor decomposition remains an interesting open question.



\vspace{-0.1cm}
\section{Conclusion}\label{sec:conclusion}
\vspace{-0.1cm}

We consider memory-efficient and differentially private tensor decomposition problems in this paper
and derive efficient algorithms for both online and private tensor decomposition based on the popular tensor power method framework.
Through an improved noise condition analysis of robust tensor power method,
we obtain sharper dimension-dependent sample complexity bounds for online tensor decomposition
and wider range of privacy parameters values for private tensor decomposition while still retaining utility.
Simulation results verify the tightness of our noise conditions in principle.

One important direction of future research is to extend our online and/or private tensor decomposition algorithms and analysis
to practical applications such as topic modeling and community detection,
where tensor decomposition acts as one critical step for data analysis.
An end-to-end analysis of online/private methods for these applications would be theoretically interesting
and could also greatly benefit practical machine learning of important models.

\paragraph{Acknowledgement}
A. Anandkumar is supported
in part by Microsoft Faculty Fellowship, NSF Career award CCF-1254106, ONR Award N00014-
14-1-0665, ARO YIP Award W911NF-13-1-0084 and AFOSR YIP FA9550-15-1-0221.

\clearpage

{\small
\bibliographystyle{IEEE}
\bibliography{onlinetpm}
}

\clearpage

\begin{appendices}

\section{Simulation results}\label{appsec:simulation}

\begin{figure}[t]
\centering
\includegraphics[width=4.5cm]{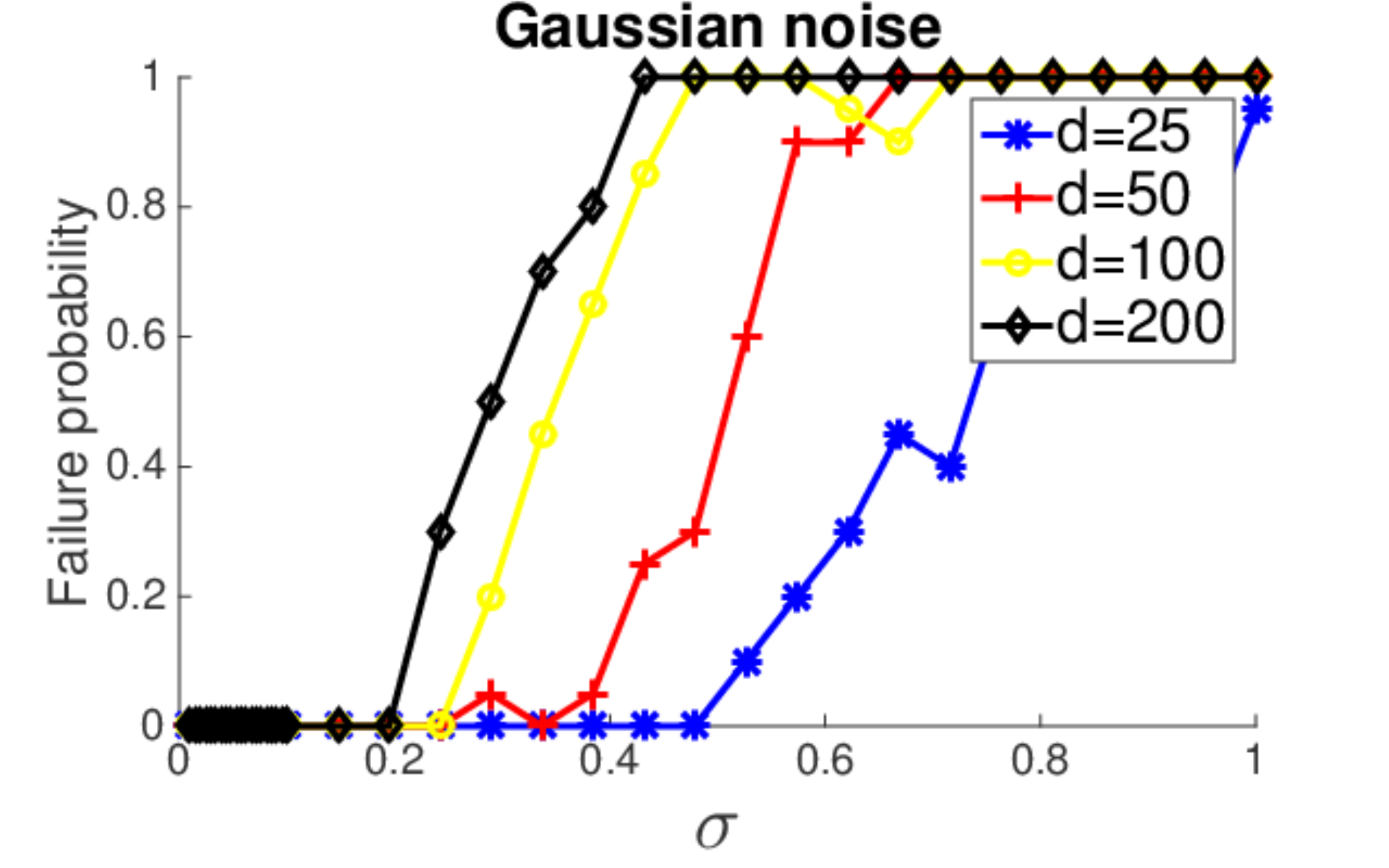}
\includegraphics[width=4.5cm]{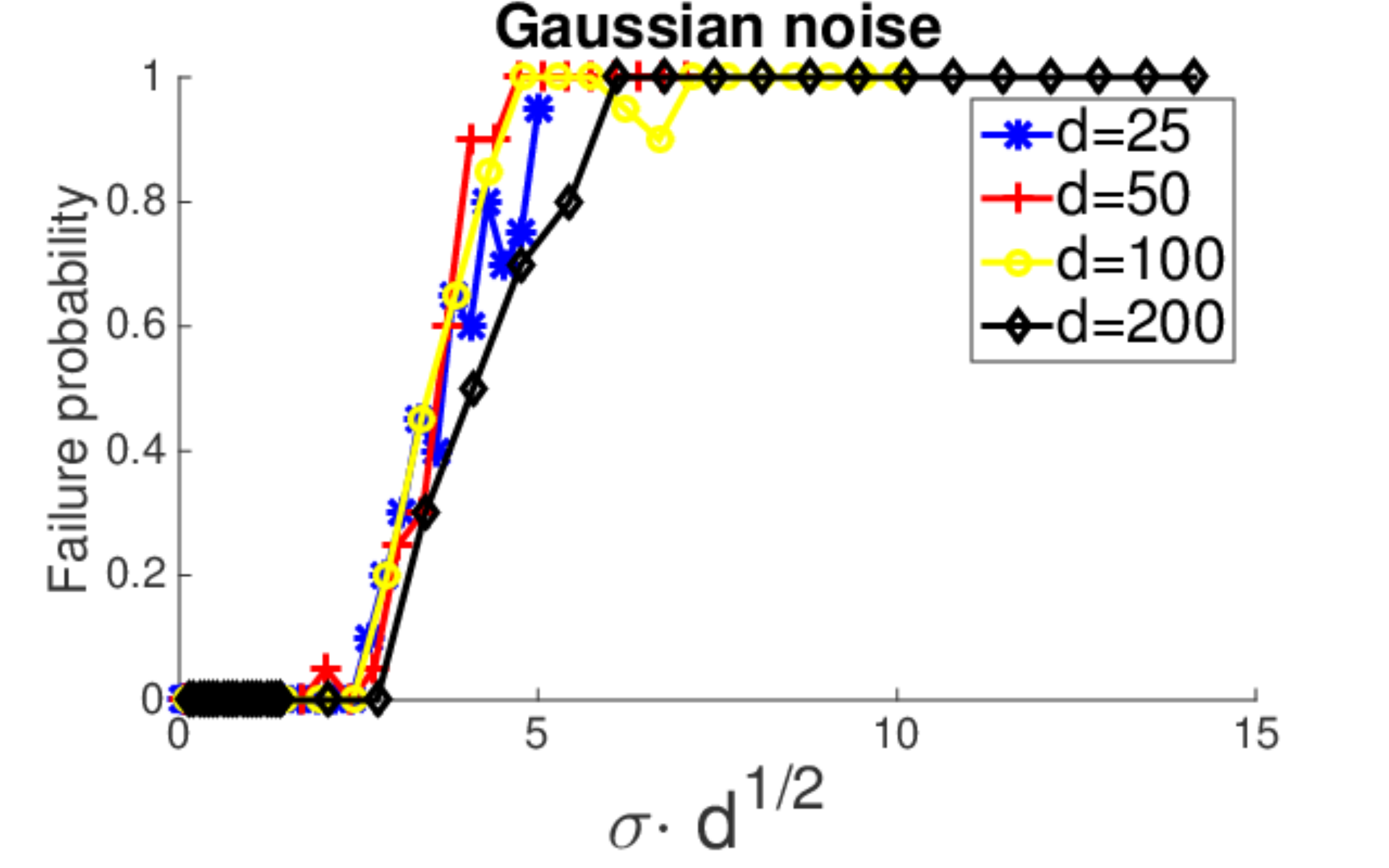}
\includegraphics[width=4.5cm]{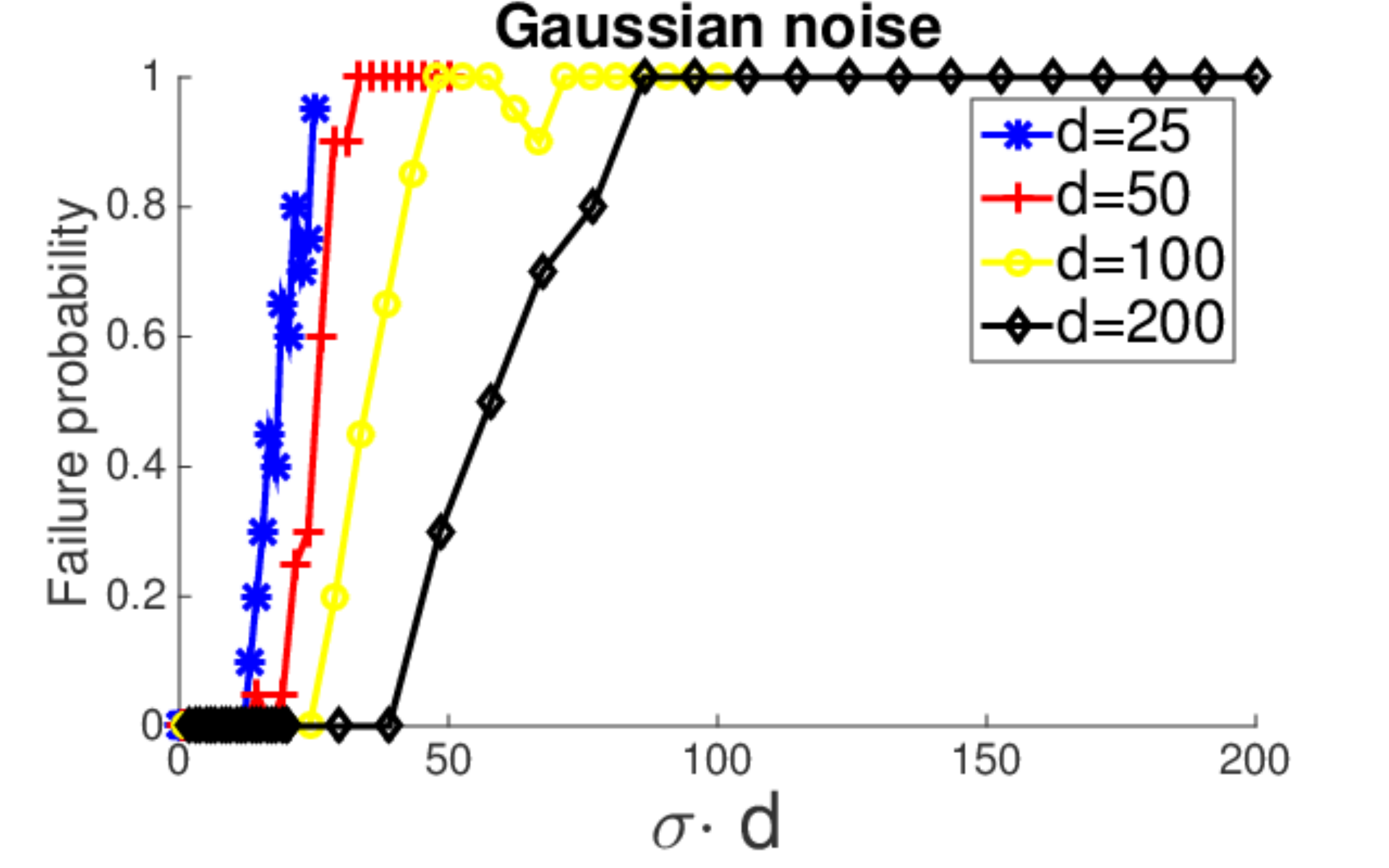}
\includegraphics[width=4.5cm]{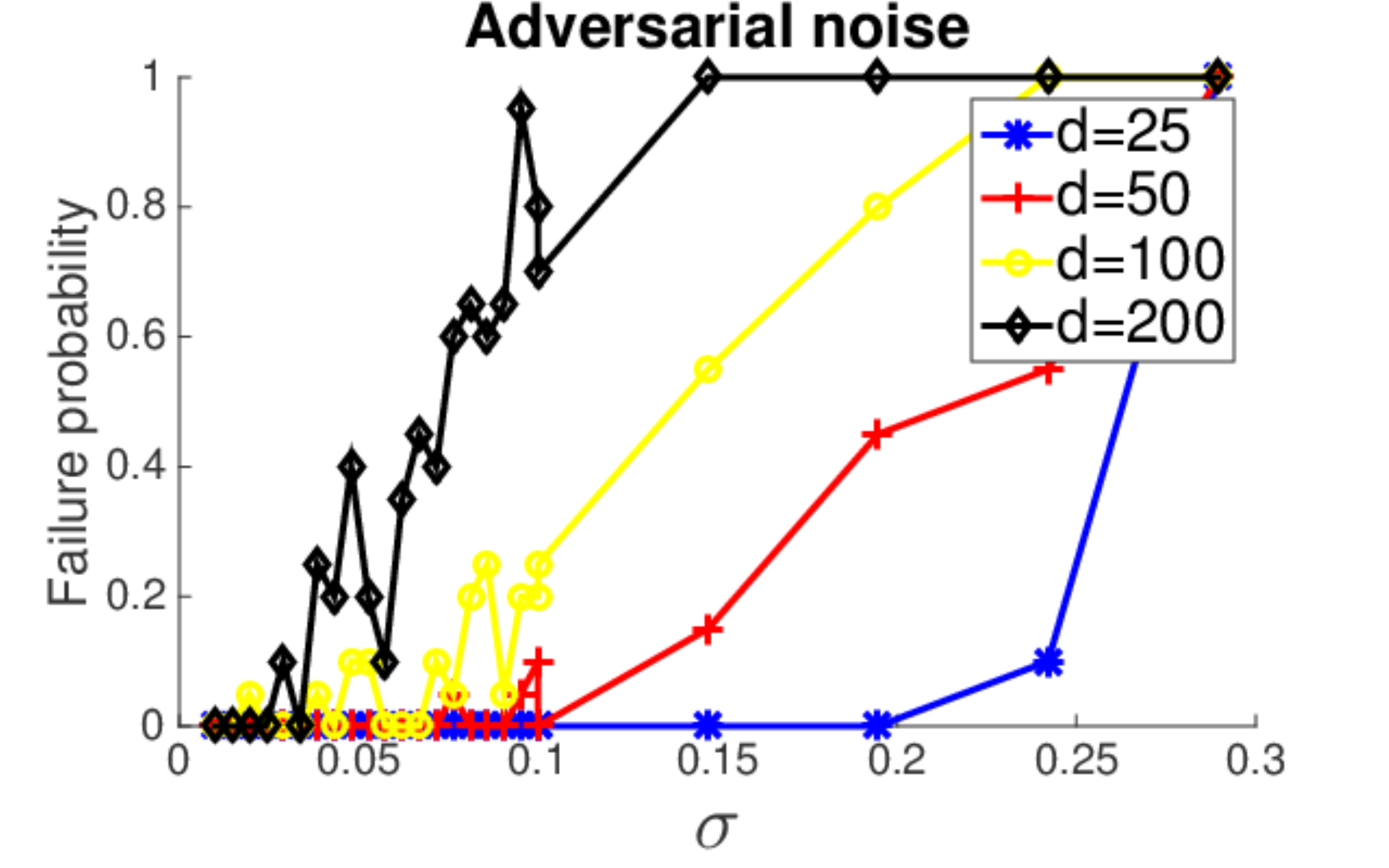}
\includegraphics[width=4.5cm]{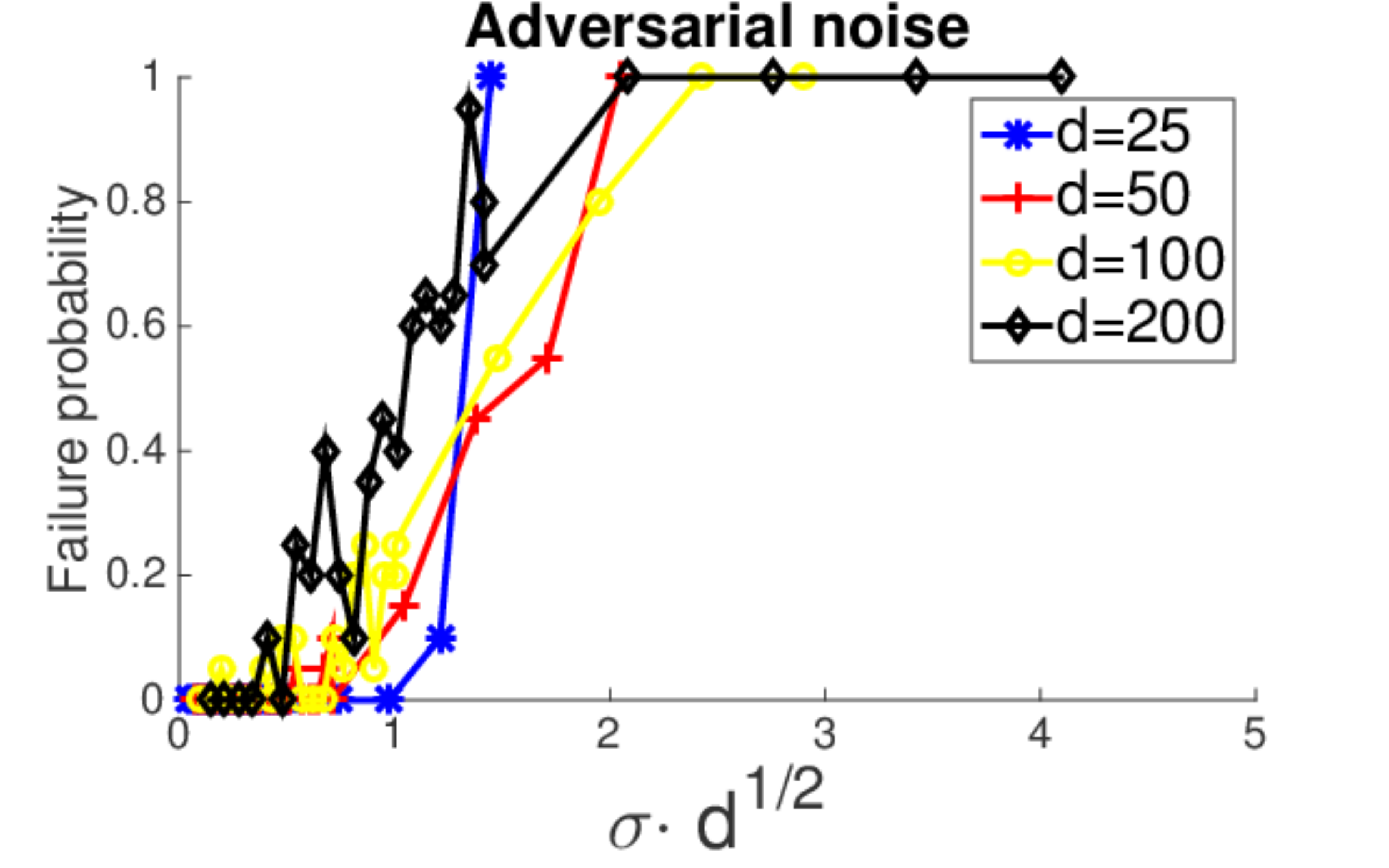}
\includegraphics[width=4.5cm]{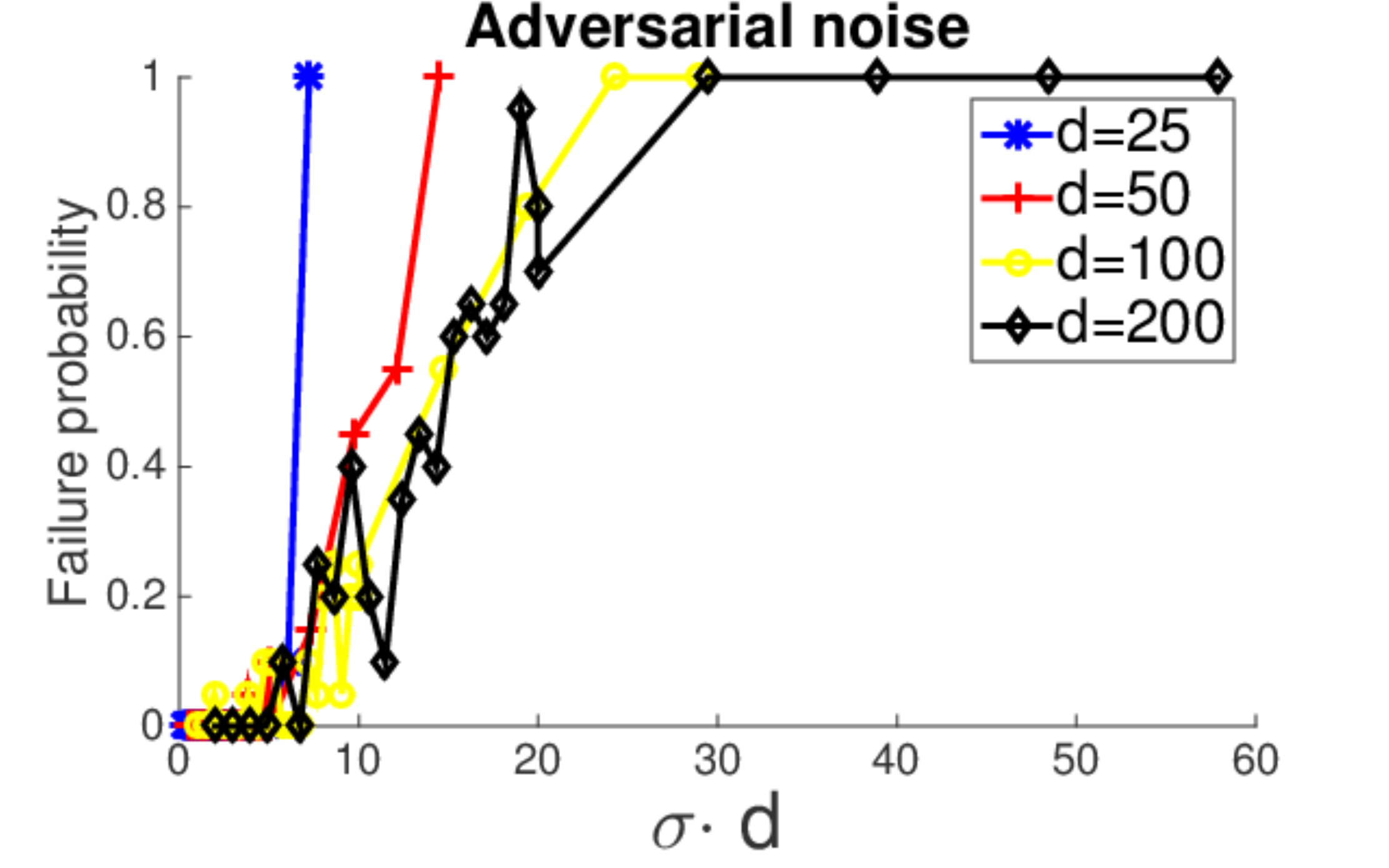}
\includegraphics[width=4.5cm]{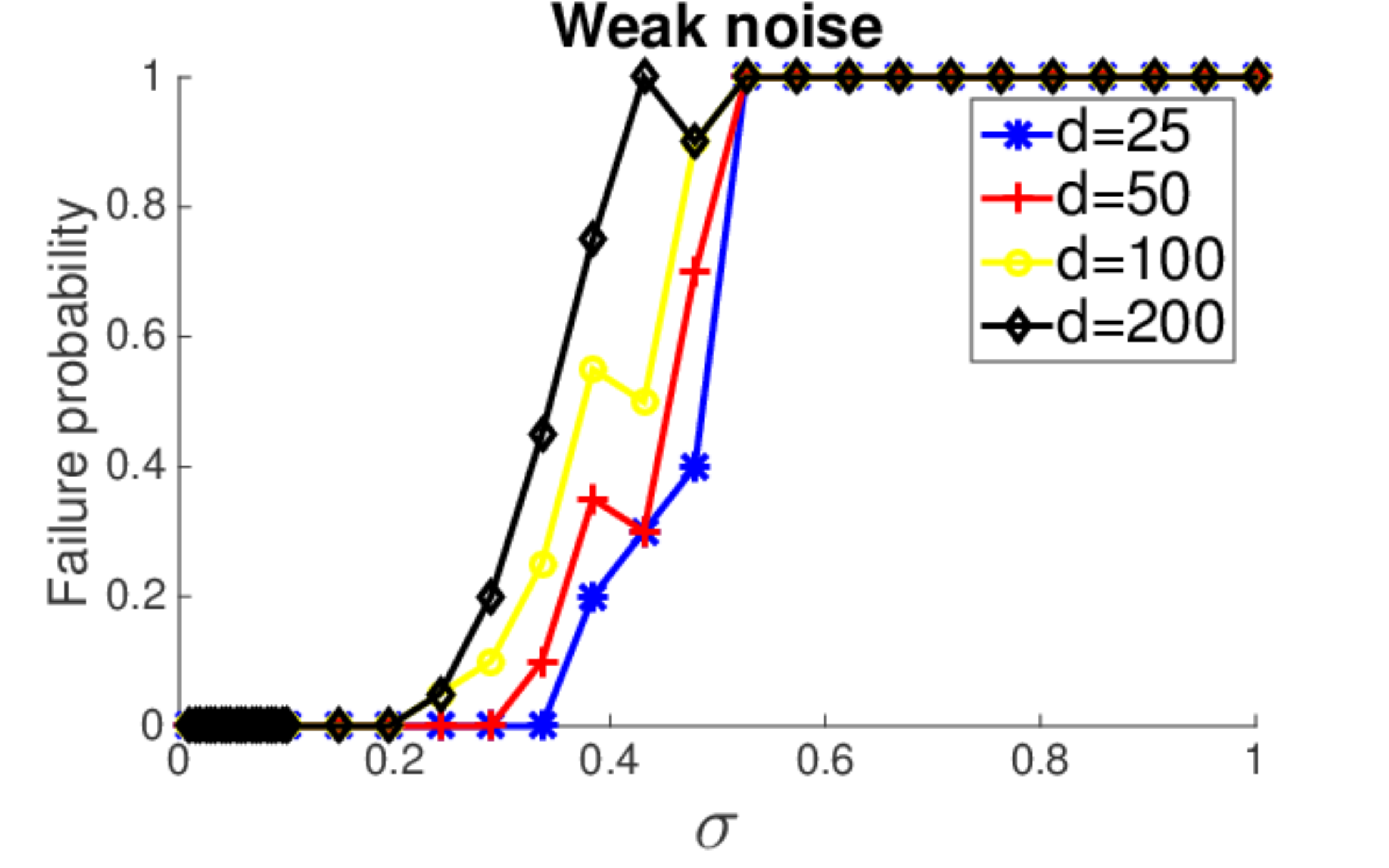}
\includegraphics[width=4.5cm]{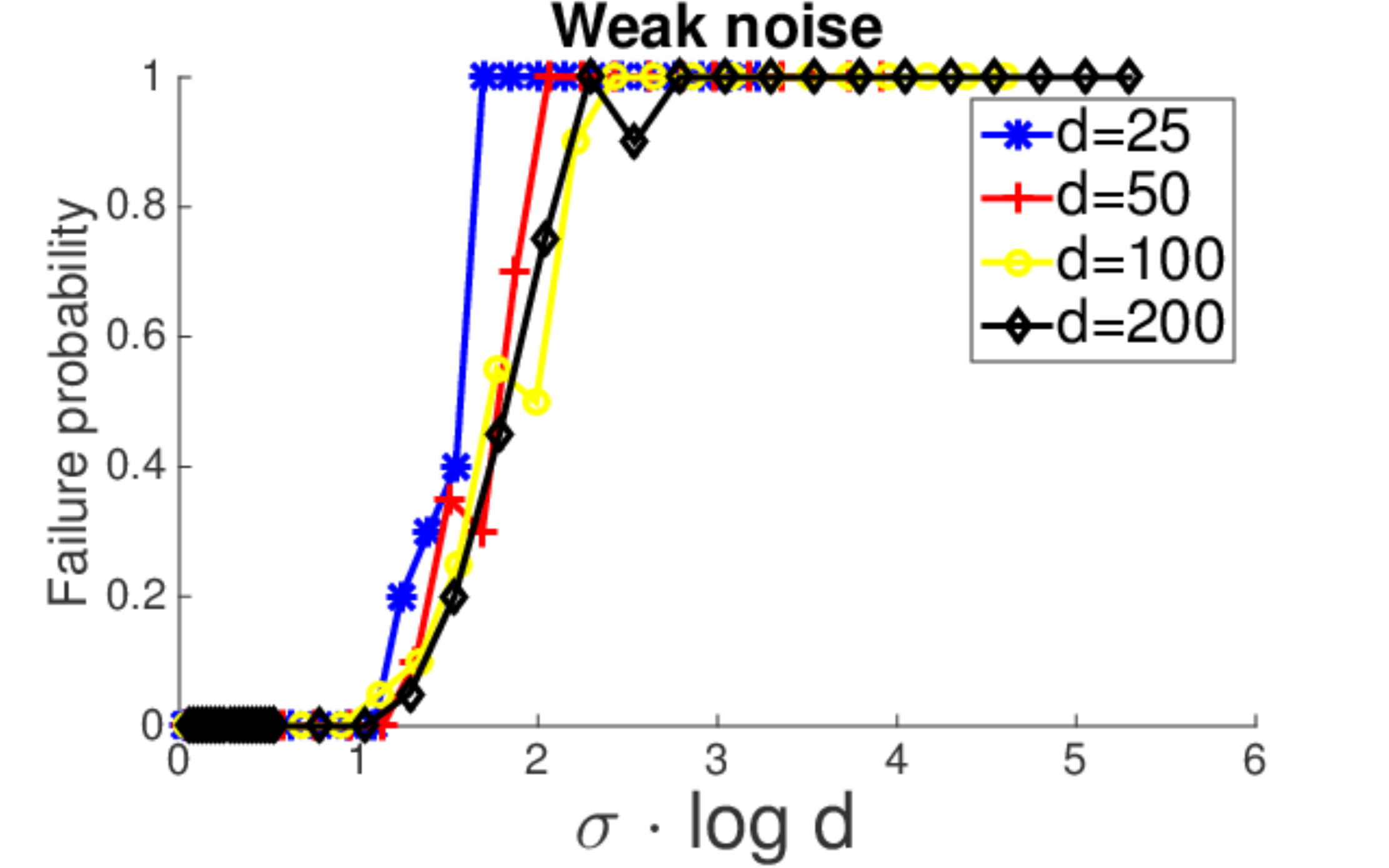}
\includegraphics[width=4.5cm]{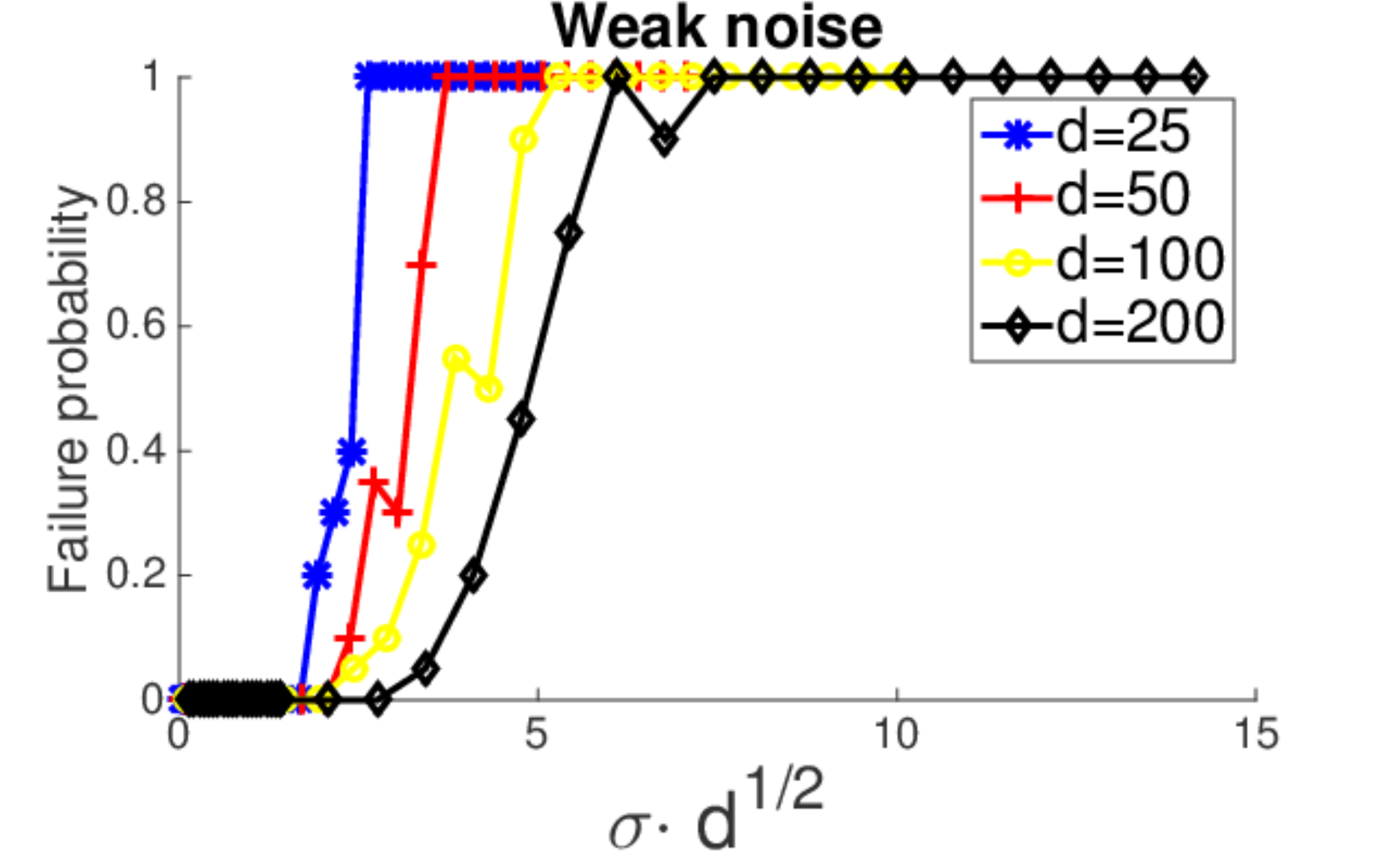}
\caption{\small
Failure probability against scaled noise magnitude on synthetic tensors.
From top to bottom rows: random Gaussian noise, adversarial noise and noise that weakly correlate with the signals (details in main text).
From left to right: scaled noise magnitudes: $\sigma$, $\sigma \sqrt{d}$, $\sigma d$ and $\sigma \ln(d)$,
labeled on each plot below the $X$ axis.}
\label{fig_synth}
\end{figure}

We verify our main theoretical results in Theorem \ref{thm:rtpm-new} on synthetic tensors.
$\mat T$ is taken to be a rank-3 tensor $\mat T=\vct v_1^{\otimes 3} + 0.75\vct v_2^{\otimes 3} + 0.5\vct v_3^{\otimes 3}$,
where $\vct v_1 = (1,0,0,0,\cdots,)$, $\vct v_2 = (0,1,0,0,\cdots)$ and $\vct v_3 = (0,0,1,0,\cdots)$.
The noise tensor $\mat E$ is synthesized according to the following three regimes:
\begin{enumerate}
\item \textbf{Random Gaussian noise}: First generate ${\mat E}_{ijk}\overset{i.i.d.}{\sim}\mathcal N(0,1)$ and then super-symmetrize $\mat E$.
\item \textbf{Adversarial Gaussian noise}: $\mat E = \sum_{i=1}^d{\vct v_2\otimes\vct e_i\otimes\vct e_i + \vct e_i\otimes\vct v_2\otimes\vct e_i + \vct e_i\otimes\vct e_i\otimes\vct v_2}$,
where $\vct e_i=(0,\cdots,0,1,0,\cdots,0)$ has all zero entries except for the $i$th one.
\item \textbf{Weakly correlated noise}: Let $\{\vct v_4,\cdots,\vct v_d\}$ be an orthonormal basis of the orthogonal complement of $\SPAN\{\vct v_1,\vct v_2,\vct v_3\}$.
Set $\mat E=\sum_{i=4}^d{\vct v_i\otimes\vct v_i\otimes\vct v_i}$.
\end{enumerate}

In Fig.~\ref{fig_synth} we plot the ``failure probability'' (measured via 20 independent trials per setting) of the robust tensor power method with random initialization
against controlled noise magnitude $\|\mat E\|_\op = \sigma$.
A trial is ``successful'' if for all $i\in\{1,2,3\}$ the recovered eigenvector $\hat{\vct v}_i$ satisfies $\hat{\vct v}_i^\top\vct v_i \geq 1/4$.
To control $\|\mat E\|_\op$, we first compute the operator norm of the generated raw noise tensor by invoking the \texttt{eig\_sshopm} routine in Matlab tensor toolbox \cite{bader2006algorithm}
(algorithm based on \cite{kolda2011shifted}) and then re-scale the entries.
By inspecting the noise levels at which phase transition of failure probabilities occurs for different tensor dimensions $d$,
ranging from 25 to 200.
It is quite clear from Fig.~\ref{fig_synth} that the phase transitions occur at $\sigma=O(1/\sqrt{d})$ for random Gaussian noise, $\sigma=O(1/d)$ for adversarial noise
and $\sigma=O(1/\log d)$ for weakly correlated noises, which matches our theoretical findings in Sec.~\ref{sec:main} up to logarithmic terms.
Our simulation results and explicit construction of an ``adversarial'' noise matrix also suggests that our analysis for robust tensor power method with random initializations
under random Gaussian noise and existing analysis for worst-case noise in \cite{anandkumar2014tensor} are tight.

\section{Comparison with whitening and matrix SVD decompositions}\label{appsec:whitening}

Another popular thread of tensor decomposition techniques involve whitening and reducing the problem to a matrix SVD decomposition,
which is very effective at reducing the dimensionality of the problem in the $k=o(d)$ undercomplete settings \cite{anandkumar2014tensor,kuleshov2015tensor,wang2015fast}.
We show in this section that \emph{without additional side information}, a standard application and analysis of tensor decomposition of whitening and matrix SVD techniques
leads to \emph{worse} error bounds than we established in Theorem \ref{thm:rtpm-new}.

When \emph{only} the 3rd-order tensor $\mat T$ is available,
one common whitening approach is to randomly ``marginalized out'' one view of $\widetilde{\mat T}$:
$$
\mat M(\vct\theta) := \widetilde{\mat T}(\mat I, \mat I,\vct\theta), \;\;\;\;\;\;
\vct\theta\text{ randomly drawn on the unit $d$-dimensional sphere;}
$$
and then evaluate top-$k$ eigen-decomposition of $\mat M(\vct\theta)$.
Let $\mathcal W=\SPAN(\vct v_1,\cdots,\vct v_k)$ be the span of the true components of $\mat T$
and $\hat{\mathcal W}$ be the top-$k$ eigenspace of matrix $\mat M(\vct\theta)$ obtained by collapsing one view of $\widetilde{\mat T}$.
We then have the following proposition that bounds the perturbation between $\mathcal W$ and $\hat{\mathcal W}$:
\begin{prop}
Suppose $\widetilde{\mat T}=\mat T+\mat\Delta_T$ as in Theorems \ref{thm:rtpm}, \ref{thm:rtpm-new} and let
$\mat\Pi_{\mathcal W},\mat\Pi_{\hat{\mathcal W}}$ be the projection operators of $\mathcal W$ and $\hat{\mathcal W}$, respectively.
Then with probability at least 0.9 over the random draw of $\vct\theta$,
$$
\left\|\mat\Pi_{\mathcal W} - \mat\Pi_{\hat{\mathcal W}}\right\|_2 \leq \widetilde O\left(\frac{\sqrt{d}\|\mat\Delta_T\|_\op}{\lambda_{\min}}\right),\;\;\;\;\;
\text{if} \;\;\;\|\mat\Delta_T\|_\op= \widetilde O\left(\frac{\lambda_{\min}}{\sqrt{d}}\right),
$$
\label{prop_whiten}
\end{prop}
\begin{proof}
First, we decompose $\mat M(\vct\theta)$ into two terms:
$$
\mat M(\vct\theta) = \sum_{i=1}^k{\lambda_i(\vct v_i^\top\vct\theta)\cdot \vct v_i\vct v_i^\top} + \mat\Delta_T(\mat I,\mat I,\vct\theta).
$$
Define $\bar\lambda_i=\lambda_i(\vct v_i^\top\vct\theta)$ and $\bar{\mat E}=\mat\Delta_T(\mat I,\mat I,\vct\theta)$.
We then have that
$$
\mat M(\vct\theta) = \mat M_0+\bar{\mat E},
$$
where $\mat M_0$ is a $d\times d$ rank-$k$ matrix with eigenvalues $(\bar\lambda_1,\cdots,\bar\lambda_k)$ and eigenvectors $(\vct v_1,\cdots,\vct v_k)$,
and $\bar{\mat E}$ satisfies $\|\bar{\mat E}\|_2\leq \|\mat\Delta_T\|_\op$.
Since $\vct\theta$ is uniformly sampled from the $d$-dimensional unit sphere,
by standard concentration arguments we have that $|\vct v_j^\top\vct\theta|=\widetilde\Omega(1/\sqrt{d})$ with overwhelming probability for all $j=1,\cdots,k$ and hence
$$
\sigma_k(\mat M_0) = \widetilde\Omega(\lambda_{\min}/\sqrt{d}),
$$
where $\sigma_k(\cdot)$ denotes the $k$th largest singular value of a matrix.
Applying Weyl's theorem (Lemma \ref{lem_weyl}) we have that
$$
\sigma_k(\mat M(\vct\theta)) \geq \sigma_k(\mat M_0) - \|\bar{\mat E}\|_2 =  \widetilde\Omega(\lambda_{\min}/\sqrt{d}),
$$
where the last inequality is due to the condition imposed on noise magnitude $\|\mat\Delta_T\|_\op$ and the fact that $\|\bar{\mat E}\|_2\leq\|\mat\Delta_T\|_\op$.
Applying Wedin's theorem (Lemma \ref{lem_wedin}) with $\alpha=0$ and $\delta = \sigma_k(\mat M(\vct\theta))=\widetilde\Omega(\lambda_{\min}/\sqrt{d})$ we arrive at
$$
\left\|\mat\Pi_{\mathcal W}-\mat\Pi_{\hat{\mathcal W}}\right\|_2 \leq \frac{\|\bar{\mat E}\|_2}{\sigma_k(\mat M(\vct\theta))} \leq \widetilde O\left(\frac{\sqrt{d}\|\mat\Delta_T\|_\op}{\lambda_{\min}}\right).
$$
\end{proof}
This simple result shows that the whitening trick does not trivially lead to matching noise conditions in Theorem \ref{thm:rtpm-new} under $k=o(d)$ settings.

\section{Proof of Theorem \ref{thm:rtpm-new}}\label{appsec:main-proof-rtpm}

\subsection{Proof sketch of Theorem \ref{thm:rtpm-new}}\label{subsec:proof}

In this section we sketch the proof of Theorem \ref{thm:rtpm-new}.
Our proof is mostly built upon the analysis in \cite{anandkumar2014tensor} for robust tensor power method.
However, we borrow new ideas from \cite{hardt2014noisy} to substantially revise the per-iteration analysis (Lemma \ref{lem_main}),
which subsequently results in desired relaxation of noise conditions.
Some results and arguments in \cite{anandkumar2014tensor}, especially those involved with absolute constants, are simplified
for accessibility purposes.

We start with Lemma \ref{lem_init} that analyzes random initializations against eigenvectors.
\begin{lem}
Fix $j^*\in\{1,\cdots,k\}$ and $\eta\in(0,1/2)$. Suppose $L$ satisfies $L=\Omega(k/\eta)$.
Then with probability at least $1-\eta$ there exists a initialization $\vct u_0$ such that
\begin{equation}
\max_{1\leq j\leq k,j\neq j^*} |\vct v_j^\top\vct u_0| \leq 0.5 |\vct v_{j^*}^\top\vct u_0| \;\;\;\;\text{and}\;\;\;\;
|\vct v_{j^*}^\top\vct u_0| \geq 1/\sqrt{d}.
\label{eq_init}
\end{equation}
\label{lem_init}
\end{lem}
Roughly speaking, Lemma \ref{lem_init} shows that with $L=\Omega(d\log d)$ initializations
the initial vector $\vct u_0$ will slightly bias towards one of the directions $j^*$ with overwhelming probability.
The lemma is a slight generalization of Lemma B.1 in \cite{anandkumar2014tensor} to the $k\leq d$ case
and their proofs are similar.
For completeness purposes we include its proof in Appendix \ref{appsec:proof-rtpm}.
Applying a standard boosting argument we have the following corollary, which guarantees exponentially decaying failure probabilities:
\begin{cor}
For any $\tilde\eta\in(0,1/2)$, with $L=\Omega(k\log(1/\tilde\eta))$ initializations Eq.~(\ref{eq_init}) holds for at least one initialization
with probability at least $1-\tilde\eta$.
\label{cor_init}
\end{cor}

The following lemma is the key lemma that characterizes the recovery of \emph{single} eigenvectors of the robust tensor power method.
\begin{lem}
Suppose $\lambda_1\geq\lambda_2\geq\cdots\geq\lambda_k\geq 0$ and assume without loss of generality that the conditions in Lemma \ref{lem_init} hold with respect to $j^*=1$.
Assume in addition that
\begin{equation*}
\|\mat\Delta_T(\mat I,\vct u_t,\vct u_t)\|_2 \leq \min\left\{\tilde\epsilon_t, \frac{\lambda_1}{40\sqrt{d}}\right\},\;\;
|\mat\Delta_T(\vct v_j,\vct u_t,\vct u_t)|\leq \min\left\{\frac{\tilde\epsilon_t}{\sqrt{k}}, \frac{\lambda_1}{8d}\right\}, \;\;
\tilde\epsilon_t\leq \frac{\lambda_1}{200}
\end{equation*}
for all $t\in[T]$ and $j$ such that $\lambda_j > 0$.
We then have that
\footnote{For notational simplicity, let $\tan\theta(\vct v_1,\vct u_{-1}) = \infty$.}
\begin{equation}
\max_{j\neq 1}\lambda_j|\vct v_j^\top\vct u_t|\leq 0.5\lambda_1|\vct v_1^\top\vct u_t|, \;\;\;\;\;\;
\tan\theta(\vct v_1,\vct u_t) \leq 0.8\tan\theta(\vct v_1,\vct u_{t-1}) + 8\tilde\epsilon_t/\lambda_1.
\label{eq_main_induction}
\end{equation}
In addition, if $\theta(\vct v_1,\vct u_t) \leq \pi/3$ we have further that
\begin{equation}
\frac{|\vct v_j^\top\vct u_{t+1}|}{|\vct v_1^\top\vct u_{t+1}|} \leq 0.8\frac{|\vct v_j^\top\vct u_t|}{|\vct v_1^\top\vct u_t|} + \frac{8\tilde\epsilon_t}{\lambda_1\sqrt{k}}, \;\;\;\;\;\;\forall j > 1\text{ and }\lambda_j > 0.
\label{eq_main_phase2}
\end{equation}
\label{lem_main}
\end{lem}
Compared to existing analysis in (Propositions B.1, B.2, Lemmas B.2, B.3, B.4 in \cite{anandkumar2014tensor}),
our proof in Appendix \ref{appsec:proof-rtpm} analyzes the {two-phase} behavior of robust tensor power method in a unified framework and is thus much cleaner.
Furthermore, we borrow ideas from \cite{hardt2014noisy} to prove shrinkage of the tangent angle between $\vct v_1$ and $\vct u_t$,
which subsequently leads to relaxed noise conditions.
We also prove additional bounds regarding $|\vct v_j^\top\vct u_t|$ for $j>1$ to facilitate later deflation analysis.
This result is used for relaxing noise conditions only and is hence not proved in previous work \cite{anandkumar2014tensor}.

Finally, we present the following lemma that analyzes the deflation step in the robust noisy power method,
in which both ``element-wise'' and ``full-vector'' conditions on the deflated tensor are proved.
\begin{lem}
Let $\{\hat\lambda_i,\hat{\vct v}_i\}_{i=1}^k$ be eigenvalue and (orthonormal) eigenvector pairs
that approximates $\{\lambda_i,\vct v_i\}_{i=1}^k$
with $\lambda_1\geq\cdots\geq\lambda_k > 0$ such that for all $i\in[k]$,
\begin{equation}
|\hat\lambda_i-\lambda_i|\leq C\epsilon, \;\;\;\;
\tan\theta(\vct v_i,\vct v_i)\leq \min\{\sqrt{2}, C\epsilon/\lambda_i\} \;\;\;\;
|\hat{\vct v}_i^\top\vct v_j| \leq C\epsilon/(\lambda_i\sqrt{k}),\;\;\forall j>i
\label{eq_deflation_condition}
\end{equation}
for some absolute constant $C>0$ and error tolerance parameter $\epsilon>0$.
Denote $\mat E_i=\hat\lambda_i\hat{\vct v}_i^{\otimes 3}-\lambda_i\vct v_i^{\otimes 3}$ as the $i$th reconstruction error tensor.
Let $\delta\in(0,1)$ be an arbitrary small constant.
There exist universal constants $C>0$ such that if $\epsilon\leq C'\lambda_{\min}/\sqrt{k}$ then the following holds for all $t\in[k]$ and $\|\vct u\|_2=1$:
\begin{equation}
\left\|\sum_{i=1}^t{\mat E_i(\mat I,\vct u,\vct u)}\right\|_2 \leq \kappa_t(\vct u)\epsilon \;\;\;\;\text{and}\;\;\;\;
\bigg|\sum_{i=1}^t{\mat E_i(\vct v_j,\vct u,\vct u)}\bigg| \leq \kappa_t(\vct u)^2 \frac{\epsilon}{\sqrt{k}}, \;\;\forall j>t,
\label{eq_deflation}
\end{equation}
where $\kappa_t(\vct u) = \sqrt{\delta + C''\sum_{i=1}^t{|\vct v_i^\top\vct u|^2}}$ and $C''>0$ is a universal constant.
\label{lem_deflation}
\end{lem}

We are now ready to prove the main theorem.
\begin{proof}[Proof of Theorem \ref{thm:rtpm-new}]
We use induction to prove the theorem.
For $i=1$ all conditions in Lemma \ref{lem_main} are satisfied with $\tilde\epsilon_t = 2\epsilon$
when $\epsilon\leq C_1\lambda_{\min}/\sqrt{k}$ for some sufficiently small constant $C_1>0$.
Lemma \ref{lem_main} then asserts that, with $L=\Omega(d\log d)$ initializations and $R=\Omega(\log(\lambda_1 k/\epsilon))$ iterations,
$\|\hat{\vct v}_1-\vct v_1\|_2 \leq \tan\theta(\hat{\vct v}_1,\vct v_1) \leq C_2\epsilon/\lambda_1$
for some universal constant $C_2>0$. Furthermore,
\begin{align*}
|\hat\lambda_1 - \lambda_1|
&= \big|\widetilde{\mat T}(\hat{\vct v}_1,\hat{\vct v}_1,\hat{\vct v}_1) - \lambda_1\big|
\leq \big|\mat\Delta_T(\hat{\vct v}_1,\hat{\vct v}_1,\hat{\vct v}_1)\big| + \big|\mat T(\hat{\vct v}_1,\hat{\vct v}_1,\hat{\vct v}_1) - \lambda_1\big| \\
&\leq O\left(\frac{\epsilon}{\sqrt{k}}\right) + \bigg|\lambda_1|\vct v_1^\top\hat{\vct v}_1|^3 -\lambda_1 + \sum_{j > 1}{\lambda_j|\vct v_j^\top\hat{\vct v}_1|^3}\bigg|\\
&\leq O\left(\frac{\epsilon}{\sqrt{k}}\right) + \bigg|\lambda_1\left[1+O\left(\frac{\epsilon}{\lambda_1}\right)\right] - \lambda_1 + \sum_{j>1}{\lambda_jO\left(\frac{\epsilon^3}{\lambda_j^3 k^{1.5}}\right)}\bigg|\\
&\leq O(\epsilon), \;\;\;\text{if $\epsilon\leq C_1\lambda_{\min}/\sqrt{k}$ for some sufficiently small constant $C_1$.}
\end{align*}

We next prove the theorem for the case of $i+1$ assuming by induction that the theorem holds for all $\{\lambda_j,\vct v_j\}_{j=1}^i$.
In this case, the ``new'' noise tensor $\widetilde{\mat\Delta}_T$ comes from both the original noise and also noise introduced by deflations;
that is, $\widetilde{\mat\Delta}_T=\widetilde{\mat T} + \sum_{j=1}^i{\mat E_i}$.
Invoking Lemma \ref{lem_deflation} we have that $\widetilde{\mat\Delta}_T$ satisfies conditions in Lemma \ref{lem_main} with
$$
\tilde\epsilon_t = \epsilon\left(1 + \max\{\kappa_i(\vct u_t),\kappa_i(\vct u_t)^2\}\right),
$$
where $\kappa_i(\vct u)=\sqrt{\delta+C''\sum_{j=1}^i{|\vct u^\top\vct v_j|^2}}$ as defined in Lemma \ref{lem_deflation},
provided that $\epsilon\leq C_1\lambda_{\min}/\sqrt{k}$ for some sufficiently small constant $C_1$.
Furthermore, note that for arbitrary $\delta\in(0,1)$, we can again pick $C_1'>0$ to be a sufficiently small constant (possibly depending on $\delta$) such that
$\epsilon\leq C_1'\lambda_{\min}/\sqrt{k}$ would imply $\tilde\epsilon_t\leq\min\{\lambda_1/200, 0.01\lambda_{\min}\sqrt{\delta/(C''k)}\}$.
Subsequently, by Eq.~(\ref{eq_main_induction}) we know that after $\Omega(\log(\lambda_{\max}k/\epsilon))$ iterations we have that $\tan\theta(\vct u_t,\vct v_{i+1})\leq 0.1\sqrt{\delta/(C''k)}$
and hence for any $j\leq i$, $|\vct u_t^\top\vct v_j| = \cos\theta(\vct u_t,\vct v_j) = \sin\theta(\vct u_t,\vct v_{i+1}) \leq \tan\theta(\vct u_t,\vct v_{i+1}) \leq 0.1\sqrt{\delta/(C''k)}$.
Consequently, $C''\sum_{j=1}^i{|\vct u_t^\top\vct v_j|^2} \leq 0.01\delta$ and therefore $\kappa_i(\vct u_t) \leq \sqrt{1.01\delta} \leq 1$.
We then have that $\tilde\epsilon_t\leq 2\epsilon$ and hence the resuling bounds on $|\hat{\lambda}_{i+1}-\lambda_{i+1}|$ and $\tan\theta(\vct u_t,\vct v_{i+1})$
hold with the same constant $C$ as all previous iterations before $i$.
Finally, applying Lemma \ref{lem_init} and taking a union bound over all $k$ iterations we complete the proof.
\end{proof}

\subsection{Proof of technical lemmas}\label{appsec:proof-rtpm}

\begin{proof}[Proof of Lemma \ref{lem_init}]
Let $\tilde{\vct u}_0^{(\tau)}\overset{i.i.d.}{\sim}\nml_d(0, \mat I_{d\times d})$ for $\tau\in[L]$ and
define $Z_{j,\tau} = \vct v_j^\top\tilde{\vct u}_0^{(\tau)}$ for $j\in[d]$ and $\tau\in[L]$.
Without loss of generality, assume $j^*=1$.
Consider the following sets of events:
\begin{eqnarray}
\mathcal E_1 &:=& \left\{Z: \max_{\tau\in[L]} |Z_{1,\tau}| \geq 0.5\sqrt{\ln L} - \sqrt{2\ln(6/\eta)}\right\}, \label{eq_init_e1}\\
\mathcal E_{2,\tau} &:=& \left\{Z_{\cdot,\tau}: \max_{1<j\leq k} |Z_{j,\tau}| \leq \sqrt{2\ln k} + \sqrt{2\ln(3/\eta)}\right\}, \label{eq_init_e2}\\
\mathcal E_{3,\tau} &:=& \left\{Z_{\cdot,\tau}: \sum_{j=k+1}^d{|Z_{j,\tau}|^2} \leq 3\ln(3/\eta)\cdot d+2\ln(3/\eta) \right\}. \label{eq_init_e3}
\end{eqnarray}
Suppose $\mathcal E_1$ holds with $\tau^*=\argmax_\tau{|Z_{1,\tau}|}$ and suppose in addition that
$\mathcal E_{2,\tau^*}$ and $\mathcal E_{3,\tau^*}$ hold.
To derive Eq.~(\ref{eq_init}) we need to show the following inequalities:
$$
0.5\sqrt{\ln L}-\sqrt{2\ln(6/\eta)}\;\;\; \geq \;\;\;0.5\left(\sqrt{2\ln k}+\sqrt{2\ln(3/\eta)}\right);
$$
$$
\frac{(0.6\sqrt{\ln L}-\sqrt{2\ln(6/\eta)})^2}{k\cdot (0.6\sqrt{\ln L}-\sqrt{2\ln(6/\eta)})^2 + 3\ln(3/\eta) d + 2\ln(3/\eta)}
\;\;\;\geq \;\;\;\frac{1}{d}.
$$
It can be easily verified that $L=\Omega(k/\eta)$ satisfies the above inequalities and hence imply Eq.~(\ref{eq_init})
under $\mathcal E_1\cap\mathcal E_{2,\tau^*}\cap\mathcal E_{3,\tau^*}$.

The rest of the proof is to lower bound the probabilities of events $\mathcal E_1,\mathcal E_{2,\tau^*}$ and $\mathcal E_{3,\tau^*}$.
We first consider $\mathcal E_1$.
Because $Z_{1,1},\cdots,Z_{1,L}\overset{i.i.d.}{\sim}\nml(0, 1)$ and
$f(Z_{1,1},\cdots,Z_{1,L})=\max_\tau |Z_{1,\tau}|$ is a 1-Lipschitz function, applying Lemma \ref{lem_concentration_lipschitz} we have that
\begin{equation}
\Pr\left[\max_\tau |Z_{1,\tau}| < \mu-t\right] \leq 2e^{-t^2/2},
\label{eq_init_max_concentration}
\end{equation}
where $\mu = \mathbb E[\max_\tau |Z_{1,\tau}|]$.
By Lemma \ref{lem_max_gaussian}, $\mu\geq\mathbb E[\max_\tau Z_{1,\tau}] \geq \sqrt{\ln L}/\sqrt{\pi\ln 2}\geq 0.5\sqrt{\ln L}$.
Setting $t=\sqrt{2\ln(6/\eta)}$ in Eq.~(\ref{eq_init_max_concentration}) we have that $\Pr[\mathcal E_1]\geq 1-\eta/3$.

Next, suppose $\mathcal E_1$ holds with $\tau^*=\argmax_\tau|Z_{1,\tau}|$.
Note that $\mathcal E_{2,\tau^*}$ and $\mathcal E_{3,\tau^*}$ are independent regardless of the choice of $\tau^*$,
because $Z_{1,\tau^*},\cdots,Z_{d,\tau^*}$ are independent Gaussian random variables.
We can then lower bound the probabilities of $\mathcal E_{2,\tau^*}$ and $\mathcal E_{3,\tau^*}$ separately.
We consider $\mathcal E_{2,\tau^*}$ first.
Because $Z_{2,\tau^*},\cdots,Z_{k,\tau^*}$ are i.i.d.~standard Normal random variables, applying Lemma \ref{lem_max_absolute_gaussian} we obtain
\begin{equation}
\Pr\left[\max_{2\leq j\leq k}|Z_{j,\tau^*}| > \sqrt{2\ln k} + \sqrt{2t}\right] \leq e^{-t}.
\label{eq_init_e2}
\end{equation}
Putting $t=\ln(3/\eta)$ in Eq.~(\ref{eq_init_e2}) we have that $\Pr[\mathcal E_{2,\tau^*}|\mathcal E_1]\geq 1-\eta/3$.
For $\mathcal E_{3,\tau^*}$, it is obvious by definition that $\sum_{j=k+1}^d{|Z_{j,\tau^*}|^2}$ is a $\chi_{d-k}^2$-distributed random variable
and is independent of $\mathcal E_1$ and $\mathcal E_{2,\tau^*}$.
Applying Lemma \ref{lem_chi_square} the following holds:
\begin{equation}
\Pr\left[\sum_{j=k+1}^d{|Z_{j,\tau^*}|^2} > d + 2\sqrt{dt} +2t\right] \leq e^{-t}.
\label{eq_init_e3}
\end{equation}
Putting $t=\ln(3/\eta)$ in Eq.~(\ref{eq_init_e3}) and noting that $\sqrt{d}\leq d$, $t\geq 1$, we conclude that
$\Pr[\mathcal E_{3,\tau^*}|\mathcal E_{1}]\geq 1-\eta/3$.
Finally, applying union bound we have that $\Pr[\mathcal E_1\cap\mathcal E_{2,\tau^*}\cap\mathcal E_{3,\tau^*}]\geq 1-\eta$.
\end{proof}

\begin{proof}[Proof of Lemma \ref{lem_main}]
First, as a consequence of Corollary \ref{cor_init}, we know that $|\vct v_1^\top\vct u_0| \geq 1/\sqrt{d}$.
The conditions in Lemma \ref{lem_main} then imply $|\mat\Delta_T(\vct v_j,\vct u_t,\vct u_t)| \leq \lambda_1|\vct v_1^\top\vct u_0|^2/8$.
We now use induction to prove Eq.~(\ref{eq_main_induction}).
When $t=0$ Eq.~(\ref{eq_main_induction}) trivially holds due to Lemma \ref{lem_init} and the condition that $j^*=1$ corresponds to the largest eigenvalue $\lambda_1$.
The objective is then to prove Eq.~(\ref{eq_main_induction}) for the case of $t+1$, assuming it holds for all iterations up to $t$.

We first consider the second part of Eq.~(\ref{eq_main_induction}) concerning $\tan\theta(\vct v_1,\vct u_t)$.
Let $\mat V\in\mathbb R^{d\times (k-1)}$ be an orthonormal basis of the complement subspace $\mathcal V_\perp = \SPAN\{\vct v_2,\cdots,\vct v_k\}$.
Further let $\vct\varepsilon_t = \mat\Delta_T(\mat I,\vct u_t,\vct u_t)$.
Because $\mat T(\mat I,\vct u_t,\vct u_t)$ lies in the span of columns of $\mat V$, we have that
\begin{equation*}
\tan\theta(\vct v_1,\vct u_{t+1})
\leq \frac{\|\mat V^\top\mat T(\mat I,\vct u_t,\vct u_t)\|_2 + \|\vct\varepsilon_t\|_2}{|\vct v_1^\top[\mat T(\mat I,\vct u_t,\vct u_t)+\vct\varepsilon_t]|}\\
\leq \frac{\|\mat V^\top\mat T(\mat I,\vct u_t,\vct u_t)\|_2 + \|\vct\varepsilon_t\|_2}{|\vct v_1^\top\mat T(\mat I,\vct u_t,\vct u_t)| - |\vct v_1^\top\vct\varepsilon_t|}.
\end{equation*}
In addition, note that
\begin{equation*}
\|\mat V^\top\mat T(\mat I,\vct u_t,\vct u_t)\|_2
= \sqrt{\sum_{j=2}^k{\lambda_j^2|\vct v_j^\top\vct u_t|^4}}
\leq \max_{j\neq 1}\lambda_j|\vct v_j^\top\vct u_t|\cdot \sqrt{\sum_{j=2}^{k}{|\vct v_j^\top\vct u_t|^2}},
\end{equation*}
where the first equality is due to the orthogonality of $\{\vct v_2,\cdots,\vct v_k\}$ and in the last inequality we apply H\''{o}lder's inequality.
Because $\sqrt{\sum_{j=2}^k{|\vct v_j^\top\vct u_t|^2}} = \|\mat V^\top\vct u_t\|_2$, we have that
\begin{align}
&\tan\theta(\vct v_1,\vct u_{t+1})
\leq \frac{\|\mat V^\top\vct u_t\|_2\cdot \max_{j\neq 1}\lambda_j|\vct v_j^\top\vct u_t| + \|\vct\varepsilon_t\|_2}{|\vct v_1^\top\vct u_t|\cdot \lambda_1|\vct v_1^\top\vct u_t| - |\vct v_1^\top\vct\varepsilon_t|}\nonumber\\
&= \tan\theta(\vct v_1,\vct u_t)\left[\frac{\max_{j\neq 1}\lambda_j|\vct v_j^\top\vct u_t| + {\|\vct\varepsilon_t\|_2}/{\|\mat V^\top\vct u_t\|_2}}{\lambda_1|\vct v_1^\top\vct u_t| - {|\vct v_1^\top\vct\varepsilon_t|}/{|\vct v_1^\top\vct u_t|}}\right] \nonumber\\
&\leq \tan\theta(\vct v_1,\vct u_t)\left[\frac{0.5\lambda_1|\vct v_1^\top\vct u_t|+ {\|\vct\varepsilon_t\|_2}/{\|\mat V^\top\vct u_t\|_2}}{\lambda_1|\vct v_1^\top\vct u_t| - {|\vct v_1^\top\vct\varepsilon_t|}/{|\vct v_1^\top\vct u_t|}}\right] \label{eq_line3_main}\\
&= \tan\theta(\vct v_1,\vct u_t)\underbrace{\left[\frac{1}{2} \frac{1}{1-|\vct v_1^\top\vct\varepsilon_t|/(\lambda_1|\vct v_1^\top\vct u_t|^2)}\right]}_{\alpha}
+ \underbrace{\frac{1}{1-|\vct v_1^\top\vct\varepsilon_t|/(\lambda_1|\vct v_1^\top\vct u_t|^2)}}_{2\alpha}\cdot \underbrace{\frac{\|\vct\varepsilon_t\|_2}{\lambda_1|\vct v_1^\top\vct u_t|^2}}_{\beta}. \nonumber
\end{align}
Here in Line~\ref{eq_line3_main} we apply the induction hypothesis that $\max_{j\neq 1}{\lambda_j|\vct v_j^\top\vct u_t|}\leq 0.5\lambda_1|\vct v_1^\top\vct u_t|$.
Before proceeding the analysis we first show that $|\vct v_1^\top\vct u_0|\leq |\vct v_1^\top\vct u_t|$.
Applying the induction hypothesis, we have that
$$
\tan\theta(\vct v_1,\vct u_t) \leq 0.8^t\tan\theta(\vct v_1,\vct u_0) + 40\tilde\epsilon_t/\lambda_1 \leq 0.8\tan\theta(\vct v_1,\vct u_0)+40\tilde\epsilon_t/\lambda_1 \leq \tan\theta(\vct v_1,\vct u_0),
$$
where the last inequality is due to $\tilde\epsilon_t \leq \lambda_1/200$.
Subsequently, $\theta(\vct v_1,\vct u_t)\leq \theta(\vct v_1,\vct u_0)$ and hence
$|\vct v_1^\top\vct u_t| = \cos\theta(\vct v_1,\vct u_t)\geq \cos\theta(\vct v_1,\vct u_0)=|\vct v_1^\top\vct u_0|$.
Now applying $|\vct v_1^\top\vct\varepsilon_t|\leq|\vct v_1^\top\vct u_0|^2/4$ we obtain
\begin{equation}
|\vct v_1^\top\vct\varepsilon_t| \leq \frac{\lambda_1|\vct v_1^\top\vct u_0|^2}{4} \leq \frac{\lambda_1|\vct v_1^\top\vct u_t|^2}{4}\;\; \Longrightarrow\;\;\frac{1}{1-|\vct v_1^\top\vct\varepsilon_t|/(\lambda_1|\vct v_1^\top\vct u_t|^2)} \leq \frac{3}{2}
\;\; \Longrightarrow \;\;
\alpha\leq \frac{3}{4}.
\label{eq_cond1_lemma_key}
\end{equation}
Next we bound $\beta$ by considering two cases.
In the first case of $|\vct v_1^\top\vct u_t|\leq 0.5$, we have that
\begin{equation}
\beta 
= \tan\theta(\vct v_1,\vct u_t)\cdot\frac{\|\mat V^\top\vct\varepsilon_t\|_2}{\lambda_1|\vct v_1^\top\vct u_t|\sqrt{1-|\vct v_1^\top\vct u_t|^2}}
\leq \frac{2\|\vct\varepsilon_t\|_2}{\lambda_1|\vct v_1^\top\vct u_t|}\cdot\tan\theta(\vct v_1,\vct u_t)
\leq 0.05\tan\theta(\vct v_1,\vct u_t).
\label{eq_cond3_lemma_key}
\end{equation}
where the last inequality is due to the condition that $\|\vct\varepsilon_t\|_2 \leq \frac{\lambda_1|\vct v_1^\top\vct u_0|}{40} \leq \frac{\lambda_1|\vct v_1^\top\vct u_t|}{40}$.
On the other hand, if $|\vct v_1^\top\vct u_t| > 0.5$ the following holds:
\begin{equation}
\beta = \frac{\|\vct\varepsilon_t\|_2}{\lambda_1|\vct v_1^\top\vct u_t|^2}
\leq \frac{4\|\vct\varepsilon_t\|_2}{\lambda_1} \leq \frac{4\tilde\epsilon_t}{\lambda_1}.
\label{eq_cond4_lemma_key}
\end{equation}
Combining Eq.~(\ref{eq_cond1_lemma_key},\ref{eq_cond3_lemma_key},\ref{eq_cond4_lemma_key})
we obtain $\tan\theta(\vct v_1,\vct u_{t+1}) \leq 0.8\tan\theta(\vct v_1,\vct u_t) + 8\tilde\epsilon_t/\lambda_1$.

We next prove the first part of Eq.~(\ref{eq_main_induction}), namely that $\max_{j\neq 1}\lambda_j|\vct v_j^\top\vct u_{t+1}|\leq 0.5\lambda_1|\vct v_1^\top\vct u_{t+1}|$.
For those $j$ with $\lambda_j=0$ the bound trivially holds.
So we consider only $j>1$ with $\lambda_j>0$.
We then have that
\begin{equation*}
\frac{\lambda_1|\vct v_1^\top\vct u_{t+1}|}{\lambda_j|\vct v_j^\top\vct u_{t+1}|}
= \frac{\lambda_1|\vct v_1^\top[\mat T(\mat I,\vct u_t,\vct u_t) + \vct\varepsilon_t]|}{\lambda_j|\vct v_j^\top[\mat T(\mat I,\vct u_t,\vct u_t) + \vct\varepsilon_t]|}
\geq \underbrace{\left(\frac{\lambda_1|\vct v_1^\top\vct u_t|}{\lambda_j|\vct v_j^\top\vct u_t|}\right)^2}_{\alpha'}\cdot
\underbrace{\frac{1-|\vct v_1^\top\vct\varepsilon_t|/(\lambda_1|\vct v_1^\top\vct u_t^2|)}{1+|\vct v_j^\top\vct\varepsilon_t|/(\lambda_j|\vct v_j^\top\vct\varepsilon_t|^2)}}_{\beta'}.
\end{equation*}
By induction hypothesis $\alpha'\geq 4$. Applying conditions on $|\vct v_1^\top\vct\varepsilon_t|$ we get $|\vct v_1^\top\vct\varepsilon_t| \leq \frac{\lambda_1|\vct v_1^\top\vct u_0|^2}{4} \leq \frac{\lambda_1|\vct v_1^\top\vct u_t|^2}{4}$
and hence $|\vct v_1^\top\vct\varepsilon_t|/(\lambda_1|\vct v_1^\top\vct u_t|^2) \leq 1/4$.
On the other hand,
$$
\left(\frac{\lambda_1|\vct v_1^\top\vct u_t|}{\lambda_j|\vct v_j^\top\vct u_t|}\right)^2\left[1+\frac{|\vct v_j^\top\vct\varepsilon|}{\lambda_j|\vct v_j^\top\vct u_t|^2}\right]^{-1}
= \left[\left(\frac{\lambda_j|\vct v_j^\top\vct u_t|}{\lambda_1|\vct v_1^\top\vct u_t|}\right)^2+\frac{\lambda_j|\vct v_j^\top\vct\varepsilon_t|}{\lambda_1^2|\vct v_1^\top\vct u_t|^2}\right]^{-1}
\geq \left[\frac{1}{4}+ \frac{|\vct v_j^\top\vct\varepsilon_t|}{\lambda_1|\vct v_1^\top\vct u_t|^2}\right]^{-1}.
$$
Because $|\vct v_j^\top\vct\varepsilon_t| \leq \frac{\lambda_1|\vct v_1^\top\vct u_0|^2}{8} \leq \frac{\lambda_1|\vct v_1^\top\vct u_t|^2}{8}$,
the right-hand side of the above equation is lower bounded by $8/3$.
Therefore, $\alpha'\beta' \geq \frac{8}{3}(1-\frac{1}{4}) \geq 2$.

The last part of this proof is devoted to showing Eq.~(\ref{eq_main_phase2}).
Under the condition that $\theta(\vct v_1,\vct u_t)\leq \pi/3$ we have that $\cos\theta(\vct v_1,\vct u_t) = |\vct v_1^\top\vct u_t| \geq 1/2$.
Subsequently, for arbitrary $j>1$ with $\lambda_j>0$ the following holds:
\begin{equation*}
\frac{|\vct v_j^\top\vct u_{t+1}|}{|\vct v_1^\top\vct u_{t+1}|}
\leq \frac{\lambda_j|\vct v_j^\top\vct u_t|^2 + |\vct v_j^\top\vct\varepsilon_t|}{\lambda_1|\vct v_1^\top\vct u_t|^2 - |\vct v_1^\top\vct\varepsilon_t|}
\leq \frac{|\vct v_j^\top\vct u_t|}{|\vct v_1^\top\vct u_t|}\cdot \underbrace{\frac{1}{2}\frac{1}{1-|\vct v_1^\top\vct\varepsilon_t|/(\lambda_1|\vct v_1^\top\vct u_t|^2)}}_{\alpha}
 + \underbrace{\frac{|\vct v_j^\top\vct\varepsilon_t|}{\lambda_1|\vct v_1^\top\vct u_t|^2 - |\vct v_1^\top\vct\varepsilon_t|}}_{\gamma}.
\end{equation*}
Because $|\vct v_1^\top\vct u_t|\geq 1/2$ and $|\vct v_1^\top\vct\varepsilon_t| \leq \frac{1}{2}\lambda_1|\vct v_1^\top\vct u_0|^2 \leq \frac{1}{2}\lambda_1|\vct v_1^\top\vct u_t|^2$,
we have $\gamma \leq 8|\vct v_j^\top\vct\varepsilon_t|/\lambda_1$ and hence
$$
|\vct v_j^\top\vct u_{t+1}| \leq 0.8|\vct v_j^\top\vct u_t| + \frac{8|\vct v_j^\top\vct\varepsilon_t|}{\lambda_1} \leq 0.8|\vct v_j^\top\vct u_t| + \frac{8\tilde\epsilon_t}{\lambda_1\sqrt{k}}.
$$
\end{proof}

\begin{proof}[Proof of Lemma \ref{lem_deflation}]
The first part of Eq.~(\ref{eq_deflation}) is a simplified result of Lemma B.5
\footnote{Except that we operate under a $k<d$ regime, which adds no difficulty to the proof.}
in \cite{anandkumar2014tensor}
because $\|\hat{\vct v}_i-\vct v_i\|_2 \leq \tan\theta(\hat{\vct v}_i,\vct v_i)$ when $\|\hat{\vct v}_i\|_2=\|\vct v_i\|_2=1$ and $\theta<\pi/2$.
The proofs are almost identical.
So we focus on proving the second part of Eq.~(\ref{eq_deflation}) here.
Recall that $\vct v_j^\top\vct v_i=0$ for all $j>i$. Subsequently,
$$
\bigg|\sum_{i=1}^t{\mat E_i(\vct v_j,\vct u,\vct u)}\bigg|
\leq \sum_{i=1}^t{\hat\lambda_i|\vct u^\top\hat{\vct v}_i|^2|\vct v_j^\top\hat{\vct v}_i|}
\leq \frac{C\epsilon}{\sqrt{k}}\sum_{i=1}^t{\frac{\hat{\lambda}_i}{\lambda_i}|\vct u^\top\hat{\vct v}_i|^2}.
$$
Define $\hat{\vct v}_i^\perp = \hat{\vct v}_i - (\hat{\vct v}_i^\top\vct v_i)\vct v_i$ as the difference between $\hat{\vct v}_i$ and its projection on $\vct v_i$.
It is then by definition that $\|\hat{\vct v}_i^\perp\|_2 = \|\hat{\vct v}_i\|_2\sin\theta(\hat{\vct v}_i,\vct v_i) \leq \tan\theta(\hat{\vct v}_i,\vct v_i)$.
Subsequently,
\begin{multline*}
\sum_{i=1}^t{\frac{\hat{\lambda}_i}{\lambda_i}|\vct u^\top\hat{\vct v}_i|^2}
\leq \sum_{i=1}^t{\left(1+\frac{|\hat\lambda_i-\lambda_i|}{\lambda_i}\right)|\vct u^\top\hat{\vct v}_i|^2}
\leq \left(\frac{C\epsilon}{\lambda_{\min}} +1\right)\left[ \sum_{i=1}^t{\left(|\vct u^\top\vct v_i|^2 + |\vct u^\top\hat{\vct v}_i^\perp|^2\right)}\right]\\
\leq \left(\frac{C\epsilon}{\lambda_{\min}} + 1\right)\left[k\|\hat{\vct v}_i^\perp\|_2^2 + \sum_{i=1}^t{|\vct u^\top\vct v_i|^2}\right]
\leq \underbrace{\left(\frac{C\epsilon}{\lambda_{\min}} + 1\right)\frac{C^2k\epsilon^2}{\lambda_{\min}^2}}_a + \sum_{i=1}^t{|\vct u^\top\vct v_i|^2}.
\end{multline*}
Here the second step is due to H\"{o}lder inequality and the fact that $\max_{1\leq i\leq k}{\frac{|\hat\lambda_i-\lambda_i|}{\lambda_i}}\leq\frac{C\epsilon}{\lambda_{\min}}$.
For arbitrary $\delta\in(0,1)$, set$\epsilon\leq\min\{\frac{\lambda_{\min}}{C^2}, \sqrt{\frac{\delta}{2C^3}}\frac{\lambda_{\min}}{\sqrt{k}}\}\leq C'\lambda_{\min}/\sqrt{k}$
 we have that $a\leq \delta/C$, and hence
the second part of Eq.~(\ref{eq_deflation}) holds with $C''=C$.
\end{proof}

\section{Proof of results for streaming robust tensor power method}\label{appsec:proof-streaming}

\begin{proof}[Proof of Theorem \ref{thm:streaming}]
First, note that if $\vct x_1,\cdots,\vct x_n\overset{i.i.d.}{\sim} P$, $P\in\sg_D(\sigma)$ then the distribution of the sample mean $\bar{\vct x}=\frac{1}{n}\sum_{i=1}^n{\vct x_i}$
belongs to $\sg_D(\sigma/\sqrt{n})$.
To see this, fix any $\vct a\in\mathbb R^D$ and one can show that
$$
\mathbb E\left[\exp(\vct a^\top\bar{\vct x})\right]
= \prod_{i=1}^n{\mathbb E\left[\exp(\vct a^\top\vct x_i/n)\right]}
\leq \prod_{i=1}^n{\exp(\|\vct a\|_2^2\sigma^2/n^2)}
= \exp(\|\vct a\|_2^2\sigma^2/n),
$$
where the second inequality is due to the fact that $\vct x_i\in\sg_D(\sigma)$ and $\|\vct a/n\|_2^2 = \|\vct a\|_2^2/n^2$.

Under Assumptions \ref{asmp_mean}, \ref{asmp_subgaussian} and using the the above arguments, we know that
$$
\VEC(\mat\Delta_T) = \VEC\left[\frac{1}{n}\sum_{i=1}^n{\vct x_i^{\otimes 3}} - \mat T\right]\in\sg_{d^3}(\sigma/n)
$$
Now fix $\vct v_i,\vct u_t\in\mathbb R^d$ with unit $L_2$ norms.
Applying Lemma \ref{lem_subgaussian} with respect to $\mat\Sigma=\VEC(\vct v_i\otimes\vct u_t\otimes\vct u_t)\VEC(\vct v_i\otimes\vct u_t\otimes\vct u_t)^\top$ we obtain that
\begin{equation}
\Pr\left[\big|\mat\Delta_T(\vct v_i,\vct u_t,\vct u_t)\big|^2 > (1+2\sqrt{t}+t)\frac{\sigma^2}{n}\right] \leq e^{-t}, \;\;\;\;\;\forall t>0.
\label{eq_proof_streaming_eq1}
\end{equation}
Subsequently, with overwhelming probability (e.g., $\geq 1-n^{-10}$) we have that
$$
\|\mat\Delta_T(\mat I,\vct u_t,\vct u_t)\|_2 = \widetilde O\left(\sigma\sqrt{\frac{d}{n}}\right), \;\;\;\;\;\;
\big|\mat\Delta_T(\vct v_i,\vct u_t,\vct u_t)\big| = \widetilde O\left(\sigma\sqrt{\frac{1}{n}}\right).
$$
Finally, with
$$
n = \widetilde\Omega\left(\min\left\{\frac{\sigma^2d}{\epsilon^2}, \frac{\sigma^2d^2}{\lambda_{\min}^2}\right\}\right)
$$
the conditions in Eq.~(\ref{eq_deltaT}) are satisfied with overwhelming probability and hence the error bounds on $|\lambda_i-\hat{\lambda}_{\pi(i)}|$
and $\|\vct v_i-\hat{\vct v}_{\pi(i)}\|_2$.
\end{proof}

\section{Proofs of utility results for differentially private tensor decomposition}\label{appsec:proof-private}

Before proving Theorem \ref{thm_private}, we first present a lemma that upper bounds $\|\vct u_t\|_{\infty}$ when
the components $\mat V\in\mathbb R^{d\times k}$ is incoherent (Assumption \ref{asmp_coherence}) and
Gaussian noise across power updates is added.
\begin{lem}
Suppose $\mat T=\sum_{i=1}^k{\lambda_i\vct v_i^\otimes 3}$ and $\mat V=(\vct v_1,\cdots,\vct v_k)$ satisfies Assumption \ref{asmp_coherence} with coherence level $\mu_0$.
Fix $\vct u\in\mathbb R^d$ with $\|\vct u\|_2=1$ and let $\bar{\vct u}=\mat T(\mat I,\vct u,\vct u)+\sigma\cdot \vct z$,
where $\vct z\sim\nml(\vct 0,\mat I_{d\times d})$ are zero-mean independently distributed Gaussian random variables.
We then have that
$$
\frac{\|\bar{\vct u}\|_{\infty}}{\|\bar{\vct u}\|_2} = O\left(\sqrt{\frac{\mu_0 k \log d}{d}}\right).
$$
with overwhelming probability.
\label{lem_incoherent}
\end{lem}
\begin{proof}
We prove this lemma by showing an upper bound for $\|\bar{\vct u}\|_\infty$ and a lower bound on $\|\bar{\vct u}\|_2$, both with overwhelming probabilities.
For the infinity-norm upper bound, we consider the following decomposition via triangle inequality:
$$
\|\bar{\vct u}\|_\infty \leq \|\tilde{\vct u}\|_\infty + \sigma \|\vct z\|_\infty,
$$
where $\tilde{\vct u} = \mat T(\mat I,\vct u,\vct u)$ and $\vct z\sim\nml(\vct 0,\mat I_{d\times d})$.
By definition,
$$
\|\tilde{\vct u}\|_\infty = \left\|\sum_{i=1}^k{\lambda_i|\vct u^\top\vct v_i|^2\vct v_i}\right\|_\infty
= \max_{1\leq j\leq d}{\Big|\vct\lambda^\top(\mat V^\top\vct e_j)\Big|},
$$
where $\vct\lambda$ is a $k$-dimensional vector defined as $\vct\lambda=(\lambda_1|\vct u^\top\vct v_1|^2, \cdots, \lambda_k|\vct u^\top\vct v_k|^2)$.
By Cauchy-Schwarts inequality, we have that
$$
\|\tilde{\vct u}\|_{\infty} = \max_{1\leq j\leq d}{\Big|\vct\lambda^\top(\mat V^\top\vct e_j)\Big|}
\leq \|\vct\lambda\|_2\cdot \max_{1\leq j\leq d}{\|\mat V^\top\vct e_j\|_2}
\leq \sqrt{\frac{\mu_0 k}{d} \left(\sum_{i=1}^k{\lambda_i^2|\vct u^\top\vct v_i|^4}\right)},
$$
where the last inequality is due to the condition that $\mat V$ is incoherent with coherence level $\mu_0$.
In addition, $\|\vct z\|_\infty = O(\sqrt{\log d})$ with overwhelming probability, by applying Lemma \ref{lem_max_absolute_gaussian}.
As a result,
\begin{equation}
\|\bar{\vct u}\|_{\infty} \leq \sqrt{\frac{2k\mu_0}{d}\left(\sum_{i=1}^k{\lambda_i^2|\vct u^\top\vct v_i|^4}\right)} + O(\sigma\sqrt{\log d}).
\label{eq_proof_privacy_upper}
\end{equation}

We next lower bound the denominator term $\|\bar{\vct u}\|_2$.
By definition, $\bar{\vct u}$ follows a multi-variate Gaussian distribution with mean $\tilde{\vct u}$ and co-variance $\sigma^2\mat I_{d\times d}$.
Applying Lemma \ref{lem_noncentral_chi_square} with $\mu=\|\tilde{\vct u}\|_2^2/\sigma^2$ and $t=O(\log d)$ we have that
$\|\bar{\vct u}\|_2^2 = \Omega(\|\tilde{\vct u}\|_2^2 + \sigma^2d)$ with overwhelming probability.
Note also that
$$
\|\tilde{\vct u}\|_2^2 = \left\|\sum_{i=1}^k{\lambda_i|\vct u^\top\vct v_i|^2\vct v_i}\right\|_2^2 = \sum_{i=1}^k{\lambda_i^2|\vct u^\top\vct v_i|^4}
$$
because $\{\vct v_i\}_{i=1}^k$ are orthonormal vectors.
Consequently,
\begin{equation}
\|\bar{\vct u}\|_2^2 = \Omega\left(\sqrt{\sigma^2 d + \sum_{i=1}^k{\lambda_i^2|\vct u^\top\vct v_i|^4}}\right).
\label{eq_proof_privacy_lower}
\end{equation}
Combining Eqs.~(\ref{eq_proof_privacy_upper},\ref{eq_proof_privacy_lower}) we obtain
$$
\frac{\|\bar{\vct u}\|_\infty}{{\|\bar{\vct u}\|_2}}
\leq \frac{\sqrt{\frac{2\mu_0 k}{d}\sum_{i=1}^k{\lambda_i^2|\vct u^\top\vct v_i|^4}} + O(\sigma\sqrt{\log d})}{\Omega\left(\sqrt{\sigma^2 d + \sum_{i=1}^k{\lambda_i^2|\vct u^\top\vct v_i|^4}}\right)}
\leq O\left(\sqrt{\frac{\mu_0 k}{d}}\right) + O\left(\sqrt{\frac{\log d}{d}}\right) = O\left(\sqrt{\frac{\mu_0 k\log d}{d}}\right).
$$
\end{proof}

We are now ready to prove Theorem \ref{thm_private}.
\begin{proof}[Proof of Theorem \ref{thm_private}]
Applying Lemma \ref{lem_incoherent} we can with overwhelming probability upper bound the per-coordinate standard deviation of Gaussian noise calibrated in Algorithm \ref{alg_private_rtpm}:
$$
\max_{\substack{0\leq t\leq T\\1\leq\tau\leq L}}\left\{\nu\|\vct u_t^{(\tau)}\|_{\infty}^2, \nu\|\vct u_t^{(\tau)}\|_\infty^3\right\}
\leq O\left(\frac{\sqrt{K}\cdot\log(1/\delta)}{\varepsilon}\cdot {\frac{\mu_0 k\log d}{d}}\right),
$$
where $K=kL(T+1)=\widetilde O(k^2\log(\lambda_{\max} d))$.
Let $\vct\epsilon_t^{(\tau)}=\mat E(\mat I,\vct u_t^{(\tau)},\vct u_t^{(\tau)}) = \sigma_t^{(\tau)}\cdot\vct z$
be the noise vector calibrated, where $\sigma_t^{(\tau)}=\nu\|\vct u_t^{(\tau)}\|_\infty^2$.
We then have that with overwhelming probability,
$$
\|\vct\epsilon_t^{(\tau)}\|_2 = O\left(\frac{\mu_0 k^2\log(\lambda_{\max}d/\delta)}{\varepsilon\sqrt{d}}\right)
\;\;\;\;\;\;\text{and}\;\;\;\;\;\;
\big|\vct v_1^\top\vct\epsilon_t^{(\tau)}\big| = O\left(\frac{\mu_0 k^2\log(\lambda_{\max}d/\delta)}{\varepsilon d}\right).
$$
Equating the upper bound for $|\vct v_1^\top\vct\epsilon_t^{(\tau)}|$ with $O(\lambda_{\min}/d)$ we obtain the desired privacy level condition:
$$
\varepsilon = \Omega\left(\frac{\mu_0k^2\log(\lambda_{\max}d/\delta)}{\lambda_{\min}}\right).
$$
It can also be easily verified that all noise conditions in Theorem \ref{thm:rtpm-new} are satisfied with above lower bound condition on $\varepsilon$.
\end{proof}

\section{Technical lemmas}\label{appsec:tail-bound}

\subsection{Tail inequalities}

\begin{lem}[Tail bound of Lipschitz function of Gaussian random variables, \cite{cirel1976norms}]
Suppose $\vct x\sim\nml_d(0,\sigma^2\mat I_{d\times d})$ are $d$-dimensional independent Gaussian random variables
and let $f:\mathbb R^d\to\mathbb R$ be an $L$-Lipschitz function; that is,
$|f(\vct x)-f(\vct y)|\leq L\|\vct x-\vct y\|_2$ for all $\vct x,\vct y\in\mathbb R^d$.
Suppose $\mu = \mathbb E_{\vct x}[f(\vct x)]$.
Then for all $t>0$, we have that
$$
\Pr\left[\big|f(\vct x) - \mu\big| \geq t\right] \leq 2e^{-t^2/(2\sigma^2L^2)}.
$$
\label{lem_concentration_lipschitz}
\end{lem}

\begin{lem}[Bounds on maximum of Gaussian random variables, \cite{kamathbounds}]
Suppose $X_1,\cdots,X_n\overset{i.i.d.}{\sim}\nml(0,\sigma^2)$ and let $Y = \max_{1\leq i\leq n}X_i$.
We then have that
$$
\frac{\sigma}{\sqrt{\pi\ln 2}}\sqrt{\ln n}\leq \mathbb E[Y]\leq \sigma\sqrt{2}\sqrt{\ln n}.
$$
\label{lem_max_gaussian}
\end{lem}

\begin{lem}[Bounds on maximum absolute values of Gaussian random variables; Theorem 3.12, \cite{massart2007concentration}]
Suppose $X_1,\cdots,X_n\overset{i.i.d.}{\sim}\nml(0,\sigma^2)$ and let $Y=\max_{1\leq i\leq n}|X_i|$.
We then have that
$$
\Pr\left[Y \geq \sigma\sqrt{2\ln n} + \sigma\sqrt{2t}\right] \leq e^{-t}, \;\;\;\;\forall t>0.
$$
\label{lem_max_absolute_gaussian}
\end{lem}

\begin{lem}[Bounds on Chi-square random variables, \cite{laurent2000adaptive}]
Suppose $X\sim\chi_k^2$; that is, $X=\sum_{j=1}^k{Y_j^2}$ for i.i.d.~standard Normal random variables $Y_1,\cdots,Y_k$.
We then have that $\forall t>0$,
$$
\Pr\left[X \geq k+2\sqrt{kt}+2t\right] \leq e^{-t}, \;\;\;\;\;\;
\Pr\left[X \leq k-2\sqrt{kt}\right] \leq e^{-t}.
$$
\label{lem_chi_square}
\end{lem}

\begin{lem}[Bounds on non-central Chi-square random variables, \cite{birge2001alternative}]
Suppose $X\sim\chi_k^2(\mu)$; that is, $X=\sum_{j=1}^k{Y_k^2}$ for independent Normal random variables $Y_1,\cdots,Y_k$ distributed as
$Y_j\sim\nml(\mu_j,1)$, $\sum_j{\mu_j}=\mu$.
We then have that
\begin{eqnarray*}
\Pr\left[X \geq (k+\mu) + 2\sqrt{(k+2\mu)t} + 2t\right] &\leq& e^{-t}, \\
\Pr\left[X \leq (k+\mu) - 2\sqrt{(k+2\mu)t}\right] &\leq& e^{-t}.
\end{eqnarray*}
\label{lem_noncentral_chi_square}
\end{lem}

\begin{lem}[Bounds on quadratic forms of sub-Gaussian random variables, \cite{hsu2012tail}]
Suppose $X\sim\sg_D(\sigma)$ and let $\Sigma\in\mathbb R^{D\times D}$ be a positive semidefinite matrix.
Then for all $t>0$ we have that
$$
\Pr\left[X^\top\Sigma X > \sigma^2\left(\tr(\Sigma) + 2\sqrt{\tr(\Sigma^2)t} + 2\|\Sigma\|t\right)\right] \leq e^{-t}.
$$
\label{lem_subgaussian}
\end{lem}

\subsection{Matrix perturbation lemmas}

\begin{lem}[Weyl's theorem; Theorem 4.11, p.~204 in \cite{stewart1990matrix}]
Let $\mat A,\mat E$ be given $m\times n$ matrices with $m\geq n$. Then
$$
\max_{i\in[n]}\Big|\sigma_i(\mat A+\mat E)-\sigma_i(\mat A)\Big| \leq \|\mat E\|_2.
$$
\label{lem_weyl}
\end{lem}

\begin{lem}[Wedin's theorem; Theorem 4.4, pp. 262 in \cite{stewart1990matrix}]
Let $\mat A,\mat E\in\mathbb R^{m\times n}$ be given matrices with $m\geq n$.
Let $\mat A$ have the following singular value decomposition
$$
\left[\begin{array}{c}\mat U_1^\top\\ \mat U_2^\top\\ \mat U_3^\top\end{array}\right]
\mat A
\left[\begin{array}{cc} \mat V_1& \mat V_2\end{array}\right]
= \left[\begin{array}{cc} \mat\Sigma_1& \mat 0\\ \mat 0& \mat\Sigma_2\\ \mat 0& \mat 0\end{array}\right],
$$
where $\mat U_1,\mat U_2,\mat U_3,\mat V_1,\mat V_2$ have orthonormal columns and $\mat\Sigma_1$ and $\mat\Sigma_2$ are diagonal matrices.
Let $\widetilde{\mat A} = \mat A+\mat E$ be a perturbed version of $\mat A$ and
$(\widetilde{\mat U}_1,\widetilde{\mat U}_2,\widetilde{\mat U}_3,\widetilde{\mat V}_1,\widetilde{\mat V}_2,\widetilde{\mat\Sigma}_1,\widetilde{\mat\Sigma}_2)$
be analogous singular value decomposition of $\widetilde{\mat A}$.
Let $\mat\Phi$ be the matrix of canonical angles between $\range(\mat U_1)$ and $\range(\widetilde{\mat U}_1)$
and $\mat\Theta$ be the matrix of canonical angles between $\range(\mat V_1)$ and $\range(\widetilde{\mat V}_1)$.
If there exists $\alpha,\delta>0$ such that
$$
\min_i \sigma_i(\widetilde{\mat\Sigma}_1)\geq \alpha+\delta \;\;\;\;\text{and}\;\;\;\;
\max_i\sigma_i(\mat\Sigma_2)\leq\alpha,
$$
then
$$
\max\{\|\mat U_1\mat U_1^\top-\widetilde{\mat U}_1\widetilde{\mat U}_1^\top\|_2, \|\mat U_1\mat U_1^\top-\widetilde{\mat V}_1\widetilde{\mat V}_1^\top\|_2\}
 = \max\{\|\sin\mat\Phi\|_2,\|\sin\mat\Theta\|_2\} \leq \frac{\|\mat E\|_2}{\delta}.
$$
\label{lem_wedin}
\end{lem}

\subsection{Lemmas on random tensors}

\begin{lem}[Spectral norm bound of random tensors, \cite{tomioka2014spectral}]
Suppose $\mat X$ is a $p$th order tensor with dimensions $d_1,\cdots,d_p$ and each element of $\mat X$ is sampled i.i.d.~from Gaussian distribution $\nml(0,\sigma^2)$.
Then the following upper bound on $\|\mat X\|_\op$ holds with probability at least $(1-\delta)$:
$$
\|\mat X\|_\op \leq \sqrt{8\sigma^2\left(\left(\sum_{k=1}^p{d_p}\right)\ln(2K/K_0)+\ln(2/\delta)\right)},
$$
where $K_0=\ln(3/2)$.
\label{lem_tensor_spectral_norm}
\end{lem}

\end{appendices}

\end{document}